\documentclass[12pt]{article}
\usepackage[utf8]{inputenc}
\usepackage{amsmath,amssymb,amsfonts,amsthm,bm,hyperref}

\hypersetup{
    colorlinks,
    citecolor=blue,
    filecolor=blue,
    linkcolor=blue,
    urlcolor=blue,
linktocpage,linktoc=all
}

\newtheoremstyle{problem}{\topsep}{\topsep}%
     {}
     {}
     {\bfseries}
     {}
     {\newline}
     {\thmname{#1}\thmnumber{ #2} \thmnote{#3}}

\newtheorem{theorem}{Theorem}[section]
\newtheorem{definition}[theorem]{Definition}

\newtheorem{example}[theorem]{Example}
\newtheorem{lemma}[theorem]{Lemma}
\newtheorem{proposition}[theorem]{Proposition}
\newtheorem{corollary}[theorem]{Corollary}
\newtheorem{problem}[theorem]{Problem}
\newtheorem{exercise}{Exercise}[section]
\newtheorem{fact}{Fact}[section]

\def\RR{\mathbb{R}}
\def\NN{\mathbb{N}}
\def\CC{\mathbb{C}}
\def\HH{\mathbb{H}}
\def\MM{\mathbb{M}}
\def\CF{\mathcal{F}}

\def\im{\mbox{Im }}
\def\re{\mbox{Re }}
\def\det{\mbox{det}}
\def\lmin{\lambda_{\min}}
\def\lmax{\lambda_{\max}}
\def\tr{\mathrm{tr\ }}
\newcommand{\ex}[1]{\mathrm{E}\left[#1\right]}
\newcommand{\var}[1]{\mathrm{Var}\left[#1\right]}
\newcommand{\exs}[2]{\mathrm{E}_{#1}\left[#2\right]}
\newcommand{\pr}[1]{\mathrm{P}\left(#1\right)}
\newcommand{\varstat}[1]{\bm{v}\left(#1\right)}

\title{Efficient approximation of kernel functions\\ {\small Lecture notes for COV883, II Semester 2019-20}}
\author{Amitabha Bagchi\thanks{With contributions from Vinayak Rastogi, Prakhar Agarwal, Saurabh Godse, Ruturaj Mohanty, Akanshu Gupta, Vrittika Bagadia, Saumya Gupta, Shashank Goel, Hritik Bansal, Gantavya Bhatt, Kumari Rekha, Namrata Jain, Mehak, Arpan Mangal, Harsh Vardhan Jain, Gobind Singh and Pratyush Maini.}\\Computer Science and Engineering\\IIT Delhi}

\begin{document}

\maketitle

\begin{abstract}
These lecture notes endeavour to collect in one place the mathematical background required to understand the properties of kernels in general and the Random Fourier Features approximation of Rahimi and Recht (NIPS 2007) in particular. We briefly motivate the use of kernels in Machine Learning with the example of the support vector machine. We discuss positive definite and conditionally negative definite kernels in some detail. After a brief discussion of Hilbert spaces, including the Reproducing Kernel Hilbert Space construction, we present Mercer's theorem. We discuss the Random Fourier Features technique and then present, with proofs, scalar and matrix concentration results that help us estimate the error incurred by the technique. These notes are the transcription of 10 lectures given at IIT Delhi between January and April 2020.
\end{abstract}

\tableofcontents

\section{Motivation: Kernels in Machine Learning}
\label{sec:motivation}

Kernels have been found to be extremely useful in basic machine learning tasks like classification, regression and others. A full treatment of the role of kernels in machine learning can be found in the book by Scholk\"opf and Smola~\cite{scholkopf:2001}. In these notes, in order to motivate the study of kernel functions, we examine in outline the case of classification, and that too only classification using support vector machines. We will formally define kernels only in \autoref{sec:kernels}. 

\subsection{Motivating example: Support Vector Machines}

Classification is a basic problem in machine learning. In the so-called supervised setting of the 2-class classification problem we are given a {\em training set} which consists of points from a domain, each labelled with 1 or -1, denoting which class they belong to. The goal is to find an easily computable function of the domain that maps unlabelled elements of the domain (query points) to a label. In the case where the domain is the $d$-dimensional space $\RR^d$ a natural way of approaching this problem is separation by a hyperplane, a method that is attributed to Vapnik in~\cite{scholkopf:2001}. We state this formally.
\begin{problem}[Two-class classification using hyperplanes]
\label{prb:vapnik}
Given a training set $S \subseteq \RR^d \times \{-1,1\}$, find a hyperplane $\bm{w}^T\bm{x} + b = 0$ such that for
\[(\bm{x},1) \in S: \bm{w}^T\bm{x} + b \geq 0, \]
and for
\[(\bm{x},-1) \in S: \bm{w}^T\bm{x} + b < 0. \]
\end{problem}
To understand why the conditions amount to ``separation'' by a hyperplane we recall that 
$\bm{w}^{T} \bm{x}$ is actually the dot product between the vectors $\bm{w}$ and $\bm{x}$ and can be thought of as the projection of $\bm{x}$ on $\bm{w}$ scaled by the constant $c$. Hence $\bm{w}^T\bm{x} + b = 0$ is satisfied by all points for which the scaled projection has value exactly $-b$ and we can separate all points of $\RR^d$ into two sets, those on the positive ``side'' of the hyperplane and those on the negative ``side'' of it, i.e., those whose scaled projection is greater that $-b$ and those whose scaled projection is smaller. 

We note that there may be cases where there is no solution to Problem~\ref{prb:vapnik}, but for our purposes we assume we are dealing with cases where a solution exists. This problem has been approached in the literature by considering its optimization version which seeks to find a hyperplane that not only separates the two classes but has the maximum distance from every training point. This can be stated formally.
\begin{problem}[Optimal two-class classification using hyperplanes]
\label{prb:svm}
Given a training set $S \subseteq \RR^d \times \{-1,1\}$, and a constant $c > 0$,  
\[\mbox{maximize} \min_{(\bm{x}, y) \in S} y(\bm{w}^T\bm{x} + b), \]
\[\mbox{subject to}\]
\[\|\bm{w}\| = c.\]
\end{problem}
In the machine learning literature $(\bm{w}^{T} \bm{x} + b)/c$ is known as the geometric margin of the point $\bm{x}$ and  $y\cdot(w^{T} x + b)/c$, which is always positive, is known as its functional margin. The objective is to maximize the smallest functional margin. Since the scaling factor is in our control, we can restate this problem by demanding that the functional margin be fixed.
\begin{problem}[Alternate formulation of Problem~\ref{prb:svm}]
Given a training set $S \subseteq \RR^d \times \{-1,1\}$,
\[\mbox{maximize } \frac{1}{2}\|w\|^2\]
\[\mbox{subject to}\]
\[\forall (\bm{x}, y) \in S: y(\bm{w}^T\bm{x} + b) \geq 1.\]
\end{problem}
The 1/2 and the square in the objective function are to make the dual form more convenient. Going through the method of Lagrange multipliers we get the dual problem:
\begin{problem}[Dual version of Problem~\ref{prb:svm}]
\label{prb:svm-dual}
Given a training set $S \subseteq \RR^d \times \{-1,1\}$,

\[\mbox{maximize } \sum_{i=1}^{|S|} \alpha_i - \frac{1}{2} \sum_{i=1}^{|S|}\sum_{j=1}^{|S|} y_i y_j \alpha_i \alpha_j \bm{x}_i^T\bm{x}_j\]
\[\mbox{subject to}\]
\[\alpha_i \geq 0, 1 \leq i \leq |S|, \mbox{and } \sum_{i=1}^{|S|} \alpha_iy_i = 0.\]
\end{problem}
A detailed treatment of how the dual problem is derived can be seen in~\cite{ng-stanford:2019}. For our purposes we observe that the objective function in Problem~\ref{prb:svm-dual} depends entirely on pairwise dot products between training points. This leads to the following insight:
\begin{quote}
If we transform $S \subset \RR^d$ through a transformation $\kappa: \RR^d \rightarrow X$, we can find a separating hyperplane in $X$ by solving Problem~\ref{prb:svm-dual} with $\bm{x}_i^T \bm{x}_j$ replaced by $\langle \kappa(\bm{x}_i), \kappa(\bm{x}_j)\rangle$, where $\langle \cdot, \cdot\rangle$ is an inner product defined on $X$.
\end{quote}
Such a transformation could be advantageous if the training set is not separable in $\RR^d$ but its tranformed version is separable in $X$ (with hyperplanes defined via the inner product defined on $X$). Even if the transformed version is not separable, it may be more amenable to low-error classification using regularization or other techniques. 

However, computing the dot product in the transformed case could be computationally more demanding, especially if the $X$ has greater dimension than $d$. This leads us to the following conclusion
\begin{quote}
For a ``useful'' transformation $\kappa : \RR^d \rightarrow X$, if we can find an efficiently computable function $\varphi: \RR^d \times \RR^d \rightarrow \RR$ such that for all $\bm{x}, \bm{y} \in \RR^d$, 
$\varphi(\bm{x}, \bm{y}) = \langle \kappa(\bm{x}), \kappa(\bm{y})\rangle$,
then we can solve Problem~\ref{prb:svm-dual} in $X$ efficiently.
\end{quote}

We will see ahead that there is a class of functions called {\em kernels} that satisfy the condition that they can be represented as inner products in a transformed space. Whether this transformed space is useful or not depends on the application to which the machine learning method is being applied. The question of usefulness is outside the scope of these notes.

\subsection{Transforming the data: Some examples}

\begin{example}
Suppose $S \subseteq \RR^d$, let $\kappa(\bm{x}) = (\bm{x}(i)\bm{x}(j): 1 \leq i, j \leq d)$.
\end{example}
This transformation maps a $d$-dimensional vector to a $d^2$ dimensional vector comprising all the two coordinate products of the original vector. Prima facie it appears that this means that computing a dot product in the transformed space would be more expensive, $d^2$ operations as opposed to the $d$ operations required to compute a dot product in $\RR^d$. However we observe an interesting property of the transformation. Given $\bm{x}, \bm{z} \in S$, 
\begin{align*}
\phi(\bm{x})\phi(\bm{z})) & = \sum_i \sum_j \bm{x}(i)\bm{x}(j) \bm{z}(i)\bm{z}(j)\\
& =  \sum_i \bm{x}(i)\bm{z}(i) \sum_j \bm{x}(j)\bm{z}(j)\\
& = (\bm{x}^T \bm{z})^2
\end{align*}
This means that instead of performing $d^2$ operations to compute an inner product in the transformed space we can simply compute $d+1$ operations because we have ideantified a (more efficiently computable) function that gives us the value of the inner product in the transformed space.  A similar transformation is the next one:
\begin{example}
Suppose $S \subseteq \RR^2$. For some $c \in \RR$,  let $\phi(\bm{x}) =  (\bm{x}(i)\bm{x}(j): 1 \leq i, j \leq d; \sqrt{2c}\bm{x}(i): 1 \leq i \leq d; c)$
\end{example}
We omit the calculation but in this case we find that 
$$\phi(\bm{x})\phi(\bm{z}))  = (\bm{x}^T \bm{z} + c)^2$$.
We can further generalize this. 
\begin{exercise}
Find the mapping $\phi$ such that the for $\bm{x}, \bm{z} \in \RR^d$, $\phi(\bm{x})^T \phi(\bm{z}) = (\bm{x}^T \bm{z}+c)^d$.
\end{exercise}
The efficiency gained by identifying these function is of no use unless we can say that such a mapping is useful. Scholk\"opf and Smola~\cite{scholkopf:2001} claim that the kinds of transformations we have seen above, they call them ``monomial mappings'', are useful in the case of pattern analysis/classification in images where each coordinate of the data vector corresponds to a pixel. 

We now turn to the general class of functions that Mercer showed in 1909 can be represented as inner products in a transformed space: positive definite kernels.

\section{Kernels}
\label{sec:kernels}

The material in this and subsequent sections closely follows the presentation in the book by Berg, Christensen and Ressel~\cite{berg:1984}. In general we will assume that all scalars are drawn from complex numbers, $\CC$, clarifying explicitly when we are restricting to the reals. We will use $\overline{c}$ to denote the complex conjugate of $c \in \CC$. For an $n \times m$ complex matrix $A$, we will use $A^*$ to denote its conjugate transpose, i.e., $A^*_{ij} = \overline{A_{ji}}$. 

\subsection{Definitions}
\label{sec:kernels:definitions}

We first begin with some basic definitions from linear algebra.

\begin{definition}[Positive definite matrix]
\label{def:mat-pd}
An $n \times n$ matrix $A$ is called {\em positive definite} (p.d.) if 
\begin{equation}
\label{eq:mat-pd}
\sum_{j,k=1}^n c_j \overline{c_k} A_{jk} \geq 0
\end{equation}
for every $\{c_1, \ldots, c_n\} \subseteq \CC$. If the inequality \eqref{eq:mat-pd} holds strictly for every $\{c_1, \ldots, c_n\} \subseteq \CC$ then the matrix is called {\em strictly} positive definite.
\end{definition}

\paragraph{Remarks.} 
\begin{enumerate}
\item In the linear algebra literature the term positive semidefinite is often used for matrices that satisfy \eqref{eq:mat-pd} and positive definite is used for matrices that satisfy \eqref{eq:mat-pd} strictly. 
\item Implicit in Definition~\ref{def:mat-pd} is the fact that the sum on the LHS of \eqref{eq:mat-pd} is real. If it were complex then comparison with 0 would be meaningless. We recall that if $A$ is Hermitian then the LHS is guaranteed to be real.
\item If we consider the $n \times 1$ vector $\bm{c}$ such that $\bm{c}(i) = c_i$ then \eqref{eq:mat-pd} can be rewritten as $\bm{c}^* A \bm{c} \geq 0$.
\end{enumerate}

Clearly, we can define negative definite matrices by reversing the direction of the $\geq$ sign in \eqref{eq:mat-pd}, but we define an interesting class of matrices that contain the negative definite matrices.

\begin{definition}[Conditionally negative definite matrix] An $n \times n$ matrix $A$ is called {\em conditionally negative definite} (c.n.d.) if it is Hermitian and if 
\begin{equation}
\label{eq:mat-cnd}
\sum_{j,k=1}^n c_j \overline{c_k} A_{jk} \leq 0
\end{equation}
for every $\{c_1, \ldots, c_n\} \subseteq \CC$ such that $\sum_{i=1}^n c_i = 0$.
\end{definition}
Note that here we don't need to assume that the LHS of \eqref{eq:mat-cnd} is real since we have explicitly specified that $A$ is Hermitian.  
With these definitions in hand we are ready to define two interesting classes of bivariate functions. In the following we will abbreviate positive definite as p.d. and conditionally negative definite as c.n.d.

\begin{definition}[Positive definite kernel]
\label{def:ker-pd}
Let X be a non empty set. $\varphi : X \times X \rightarrow C$ is a {\em positive definite kernel} if
\begin{equation}
\label{eq:ker-pd}
\sum_{j,k=1}^{n}c_{j}\overline{c_{k}}\varphi(x_{j},x_{k}) \geq 0
\end{equation}
$\forall n \in \NN, \{x_{1},x_{2},\ldots,x_{n}\} \subseteq X$, and $\{c_{1},c_{2},\ldots,c_{n}\} \subseteq \CC$.
\end{definition}

\begin{example}
\label{eg:dot-product}
For $d > 0$, $\varphi: \RR^d \times \RR^d \rightarrow \RR$ defined as $\varphi(\bm{x}, \bm{y}) = \bm{x}^T \bm{y}$ is a p.d. kernel.
\end{example}
\begin{proof}
For some $n \in \NN$, suppose $\{\bm{x}_1, \ldots, \bm{x}_n \} \subseteq \RR^d$ and $\{c_{1},c_{2},\ldots,c_{n}\} \subseteq \CC$. Then 
\begin{align*}
\sum_{j,k = 1}^n c_j \overline{c_{k}}\varphi(\bm{x}_{j},\bm{x}_{k}) & = \sum_{j,k = 1}^n c_j \overline{c_{k}}\bm{x}_k^T \bm{x}_j = \sum_{j,k = 1}^n (c_k \bm{x}_k)^* (c_j \bm{x}_j)\\
& = \left( \sum_{j=1}^n c_j \bm{x}_j\right)^* \left( \sum_{j=1}^n c_j \bm{x}_j\right) = \left| \sum_{j=1}^n c_j \bm{x}_j \right|^2 > 0.
\end{align*}
\end{proof}
Example~\ref{eg:dot-product} shows that all bivariate functions that are dot products of a Euclidean space are p.d. kernels. Later we will see that this is true in general for the inner product of an inner product (pre-Hilbert) space.

\begin{definition}[Conditionally negative definite kernel]
\label{def:ker-cnd}
Let X be a non empty set. $\psi : X \times X \rightarrow C$ is a {\em conditionally negative definite kernel} if
it is Hermitian, i.e., $\psi(x,y) = \overline{\psi(y,x)} \forall x,y \in X$, and\begin{equation}
\label{eq:ker-cnd}
\sum_{j,k=1}^{n}c_{j}\overline{c_{k}}\varphi(x_{j},x_{k}) \leq 0
\end{equation}
$\forall n \geq 2, \{x_{1},x_{2},\ldots,x_{n}\} \subseteq X$, and $\{c_{1},c_{2},\ldots,c_{n}\} \subseteq \CC$ such that $\sum c_{i} = 0$.
\end{definition}

\begin{example}
$\psi: \RR\times \RR \rightarrow \RR$ defined as $\psi(x,y) = (x-y)^2$ is a c.n.d. kernel.
\end{example}
\begin{proof}
Since $\psi$ is symmetric so the condition that the kernel should be Hermitian is satisfied. Given $n \geq 2$, let $\{x_1, \ldots, x_n\} \subseteq \RR$ and $\{c_{1},c_{2},\ldots,c_{n}\} \subseteq \CC$ such that $\sum c_{i} = 0$. Now, 
\[
\sum_{j,k=1}^{n}c_{j}\overline{c_{k}}(x_j-x_k)^2  =  \sum_{j=1}^n c_j x_j^2 \sum_{k=1}^n \overline{c_k} + \sum_{k=1}^n \overline{c}_k x_k^2 \sum_{j=1}^n c_j - 2 \sum_{j,k=1}^n c_j \overline{c_k} x_j x_k
\]
By the condition $\sum c_i = 0$ the first two terms on the RHS are 0 and the third term can be rewritten as 
\begin{align*}
- 2 \left(\sum_{j=1}^n c_j x_j\right) \left(\overline{\sum_{j=1}^n c_j x_j}\right) & = -2 \left| \sum_{j=1}^n c_j x_j\right|^2\\
& \leq 0.
\end{align*}
\end{proof}

\paragraph{Remarks.}
\begin{enumerate}
\item If the inequality \eqref{eq:ker-pd} (resp. \eqref{eq:ker-cnd}) is strict then the kernel is called {\em strictly} positive definite (resp. strictly negative definite).
\item In both Definitions~\ref{def:ker-pd} and~\ref{def:ker-cnd} if we had imposed the restriction that all the $x_i$s are distinct we would {\em not} get a weaker definition since $$\sum_{j,k=1}^{n}c_{j}\bar{c_{k}}\varphi(x_{j},x_{k}) = \sum_{j,k=1}^{p}d_{j}\bar{d_{k}}\varphi(x_{\alpha_{j}},x_{\alpha_{k}})$$
where the $x_{\alpha_{i}}$ are distinct and $d_{k} := \sum_{i: x_i = x_{\alpha_k}} c_{i}$ for $k = 1,\ldots,p$.
\item If $\sigma : X \rightarrow X$ is a bijection then $\varphi$ is a p.d. (respectively c.n.d.) kernel iff $\varphi \circ (\sigma \times \sigma)$ is a p.d. (resp. c.n.d.) kernel. This comes from the fact that all possible finite subsets of $X$ satisfy the condition \eqref{eq:ker-pd} (resp. \eqref{eq:ker-cnd}), and so if this condition is satisfied by distinct $x_1, \ldots, x_n$ for some $n \in N$ (resp $n \geq 2$) then it is satisfied by $\sigma(x_1), \ldots, \sigma(x_n)$ which is also a set of distinct elements from $X$.
\item If X is finite with $| X | = n$, then $\varphi$ is a p.d. (respectively c.n.d.) kernel iff the $n \times n$ matrix $A$ with  $A_{ij} = \varphi(x_{j},x_{k})$, $1 \leq j,k \leq n$ is p.d. (resp. c.n.d.). This can be deduced from the fact that all the principal submatrices of a p.d. matrix are p.d. and from a similar argument for the c.n.d. case.
\item For the c.n.d. case we restrict $n$ to be at least 2 since if $n=1$ then $c_1$ will have to be 0 to satisfy the condition of summing to 0, i.e., the definition will be trivially true for all bivariate functions.
\end{enumerate}

\subsection{Some properties of p.d. and c.n.d. kernels}
\label{sec:kernels:properties}

\subsubsection{Basic properties}
\label{sec:kernels:properties:basic}
We prove some foundational properties of the classes of kernels defined above. First we show that the definition of positive defineteness for a kernel implies that the kernel is Hermitian.
\begin{proposition}[P.d. kernels are Hermitian]
\label{prp:pd-hermitian}
If $\varphi$ is a p.d. kernel defined on $X \times X$ then all its diagonal elements are positive, i.e., $\varphi(x,x) \geq 0$ for all $x \in X$ and $\varphi$ is Hermitian, i.e. $\varphi(x,y) = \overline{\varphi(y,x)}$ for all $x,y \in X$.
\end{proposition}
\begin{proof}
It is easy to see that $\varphi(x,x)$ is real and $\geq 0$ by setting $n = 1$ in Definition~\ref{def:ker-pd}. 

We now show that $\varphi$ is Hermitian. For some $x, y \in X$ such that $x \ne y$ let $A = \begin{pmatrix}
  \varphi(x,x) & \varphi(x,y)\\ 
  \varphi(y,x) & \varphi(y,y)
\end{pmatrix}$ and let  $\bm{c}_1 = \begin{pmatrix} 1 \\1\end{pmatrix}$.
Since $\varphi$ is p.d. we have that 
\[ \bm{c}_1^* A \bm{c}_1 = \varphi(x,x) + \varphi(y,y) + \varphi(x,y) + \varphi(y,x) \geq 0.\]
Since we have already shown that $\varphi(x,x)$ and $\varphi(y,y)$ are real, so this proves that $\im \varphi(x,y) = - \im \varphi(y,x)$.

Similarly, if we take
$\bm{c}_2 = 
\begin{pmatrix}
1 \\
-i
\end{pmatrix}$, since $\varphi$ is p.d. we have that 
\[ \bm{c}_2^* A \bm{c}_2 = \varphi(x,x) + \varphi(y,y) + i (\varphi(y,x) - \varphi(x,y)) \geq 0,\]
which can only be true if $\re \varphi(x,y) = \re \varphi(y,x)$.
\end{proof}
\paragraph{Remark.} As we have just seen, a p.d. kernel can be shown to be Hermitian. However the same argument doesn't hold for c.n.d. kernels since the vectors $\bm{c}_1$ and $\bm{c}_2$ don't have the property that their coordinates add up to 0. So we have to add the requirement that c.n.d. kernels are Hermitian in Definition~\ref{def:ker-cnd}.

Next we show that for real-valued kernels, symmetry and positive (resp. conditionally negative) definiteness w.r.t. real vectors is good enough to show that they are p.d. (resp. c.n.d.).
\begin{proposition}[Real-valued kernels]
\label{prp:real-kernel}
A real-valued kernel $\varphi: X \times X \rightarrow \RR$ is p.d. (resp. c.n.d.) iff $\varphi$ is symmetric, i.e., $\varphi(x,y) = \varphi(y,x)$ for all $x, y \in X$, and 
\begin{equation}
\label{eq:real-kernel}
\sum_{j,k = 1}^{n} r_j r_k \varphi(x_j, x_k) \geq 0 \;\;(\mbox{resp. } \leq 0)
\end{equation}
for all $n \in \NN$ (resp. $n \geq 2$), $\{x_1, \ldots, x_n\}\subseteq X$, $\{r_1, \ldots, r_n\} \subseteq \RR$ (resp. additionally $\sum_{i=1}^n r_i = 0$).
\end{proposition}
\begin{proof}
Given $\{x_1, \ldots, x_n\} \subseteq X$, consider $\{c_1, \ldots, c_n\} \subseteq \CC$. For each $j$, let $a_j = \re c_j$ and $b_j = \im c_j$. Then we have that
\[
\sum_{j,k = 1}^{n} c_j \overline{c_k} \phi(x_j, x_k) =  \sum_{j,k = 1}^{n} (a_j a_k + b_j b_k) \phi(x_j, x_k) + i \sum_{j,k = 1}^{n}(b_j a_k - a_j b_k) \phi(x_j, x_k)
\]
Symmetry ensures that the imaginary part of the RHS is 0. The real part of the RHS is non-negative for the p.d. case and non-positive for the c.n.d. case because it is the sum of two summations that can be written in the form of the LHS of \eqref{eq:real-kernel}.
\end{proof}

We present an important application of Proposition~\ref{prp:real-kernel}.
\begin{example}
\label{ex:gaussian-cnd}
The real-valued kernel $\psi \colon \mathbb{R} \times \mathbb{R} \rightarrow \mathbb{R}$, defined as $\psi(x,y) = (x-y)^2$ is conditionally negative definite (c.n.d.)
\end{example}

\begin{proof} Proposition~\ref{prp:real-kernel} stipulates that $\psi$ should be symmetric, which is clearly the case. So now, nor some $n$, we consider $\{x_{1},\ldots,x_{n}\} \subseteq \mathbb{R}$ and $\{c_{1},\ldots,c_{n}\} \subseteq \mathbb{R}$ s.t. $\sum_{i}c_{i}=0$. \\
\begin{align*}
    \sum_{j,k} c_{j}c_{k}(x_{j}-x_{k})^{2} 
    &= \sum_{j,k} c_{j}c_{k}\left(x_{j}^2 -x_{k}^2 - 2x_{j}x_{k} \right)\\
    &= \sum_{k}c_{k}\sum_{j}c_{j}x_{j}^2 + \sum_{j}c_j\sum_{k}c_{k}x_{k}^2 - 2\sum_{j,k}c_{j}c_{k}x_{j}x_{k}\\
    &= 0 + 0 + -2\left(\sum_{i}c_{i}x_{i}\right)^{2} \leq 0
\end{align*}
The last simplification follows from $\sum_{i}c_{i}=0$ because of which the first two terms on the RHS equal zero.
\end{proof}
We note from the proof that it is fairly straightforward to prove the same property if $\psi$ is defined as the square of the Euclidean distance between two points in $\RR^d$, i.e., the square of the euclidean distance gives us a c.n.d. kernel. Later we will see this can be extended to general Hilbert spaces. 

\subsubsection{Relating diagonal and off-diagonal elements}
\label{sec:kernels:properties:diagonal}
We now present two results that relate the diagonal elements of p.d. and c.n.d. kernels with their off-diagonal elements. 
\begin{proposition}
\label{prp:cnd-diagonal}
For a c.n.d. kernel $\psi: X \times X \rightarrow \CC$, $\psi(x,x) + \psi(y,y) \leq 2 \re \psi(x,y)$ for all $x,y \in X$.
\end{proposition}
\begin{proof}
Define $A = \begin{pmatrix}
  \psi(x,x) & \psi(x,y)\\ 
  \psi(y,x) & \psi(y,y)
\end{pmatrix}$ and let  $\bm{c} = \begin{pmatrix} 1 \\-1\end{pmatrix}$.
Since $\psi$ is c.n.d. we have that 
\[ \bm{c}^* A \bm{c}_1 = \psi(x,x) + \psi(y,y) - \psi(x,y) - \psi(y,x) \leq 0.\]
Since $\psi$ is Hermitian by definition, i.e., $\psi(y,x) = \overline{\psi(x,y)}$, the result follows. 
\end{proof}

\begin{proposition}
\label{prp:pd-diagonal}
For a p.d. kernel $\varphi: X \times X \rightarrow \CC$, $|\varphi(x,y)|^2 \leq \varphi(x,x)\varphi(y,y)$ for all $x,y \in X$.
\end{proposition}
\begin{proof}
For $x,y \in X$ define 
$A = \begin{pmatrix}
  a  & \overline{b}\\ 
  b & d
\end{pmatrix}$
where $a = \varphi(x,x)$, $d = \varphi(y,y)$, and $b = \varphi(y,x)$. Let us consider a vector $\bm{c} = \begin{pmatrix}\overline{w} \\ \overline{z}\end{pmatrix}$. We have that 
\[\begin{pmatrix}
w & z
\end{pmatrix}
\begin{pmatrix}
  a & \overline{b}\\ 
  b & d
\end{pmatrix}
\begin{pmatrix}
\overline{w} \\
\overline{z}
\end{pmatrix} = a|w|^2 + 2 \re (bz\overline{w}) + d|z|^2.
\]
When $a \ne 0$, we can rewrite the RHS as
\begin{equation}
\label{eq:pd-diagonal}
a\left|w + \dfrac{b}{a} z\right|^2 + \dfrac{|z|^2}{a}(ad - |b|^2).
\end{equation}
Since this quantity should be $\geq 0$ for all choices of $w$ and $z$, we can deduce that $a > 0$ since if this is not so we can choose $z = 0$ to contradict the fact that $\varphi$ is p.d. By a similar argument when $d \ne 0$ we can rewrite the quantity in \eqref{eq:pd-diagonal}  with the roles of $a$ and $d$ reversed to deduce that $d \geq 0$. 

Examining \eqref{eq:pd-diagonal} we note that if we choose $w = -bz/a$ the first term becomes 0, so if $\varphi$ is p.d. it is necessary that $ad - |b|^2 \geq 0$ which, taken with the fact that $a,d \geq 0$, proves the result.
\end{proof}

\subsubsection{Creating kernels from univariate functions}
\label{sec:kernels:properties:univariate}
A univariate complex-valued function can be used to create both a p.d. kernel and a c.n.d. kernel. The following proposition shows how.
\begin{proposition}
\label{prp:univariate}
Given a function $f: X \rightarrow \CC$, $\varphi(x,y) = f(x) \overline{f(y)}$ is a p.d. kernel and $\psi(x,y) = f(x) + \overline{f(y)}$ is a c.n.d. kernel.
\end{proposition}
\begin{proof}
Consider first the p.d. case. For some choice of $n \in N$, $\{x_1, \ldots, x_n\} \subseteq X$ and $\{c_1, \ldots, c_n\} \subseteq \CC$, we have 
\begin{align*}
\sum_{j=1}^n\sum_{k=1}^n c_j \overline{c_k} f(j)\overline{f(k)}&= \sum_{j=1}^n c_j f(j) \sum_{k=1}^n \overline{c_k f(k)}\\
& = \left| \sum_{j=1}^n c_j f(j) \right|^2 \geq 0.
\end{align*}
For the c.n.d. case we assume $\sum_{i=1}^n c_i = 0$. So,
\begin{align*}
\sum_{j=1}^n\sum_{k=1}^n c_j \overline{c_k} (f(j)+\overline{f(k)})&= \sum_{j=1}^n c_j f(j) \sum_{k=1}^n \overline{c_k} + \sum_{k=1}^n \overline{c_kf(k)}\sum_{j=1}^n c_j\\
& = 0
\end{align*}
\end{proof}
\begin{corollary}
A constant function $\varphi(x,x) = c$ is p.d. if and only if $c \geq 0$ and c.n.d. if and only if  $c \in \RR$.
\end{corollary}
\begin{proof}
Clearly if $c < 0$ then the function cannot be p.d. But if $c \geq 0$ then we can choose $f(x) = \sqrt{c}$ and apply Prop.~\ref{prp:univariate} to prove that $\varphi(x,y) = f(x)\overline{f(y)}$ is p.d.

For the c.n.d. case, note that a constant function cannot have a non-zero imaginary part since that would violate the condition that the function is Hermitian. For any $c \in \RR$, if we set $f(x) = s$ for any $s \in \CC$ s.t. $\re s = c/2$, then, by Prop.~\ref{prp:univariate}, $f(x) + \overline{f(y)} = c$ is c.n.d. 
\end{proof}

\subsubsection{Combining kernels}
\label{sec:kernels:properties:combining}
We now discuss the properties of kernels created by combining other kernels. First we define some terms. 

\begin{definition}[Convex cone]
Suppose $X$ is a vector space with associated scalar field $\mathbb{F} \in \{\mathbb{C}, \mathbb{R} \}$. A set $S \subseteq X$ is called a {\em convex cone} if $\forall x,y \in S$ and $\alpha, \beta \geq 0$, $\alpha x + \beta y \in S$.
\end{definition}

\begin{definition}[Pointwise convergence]
If $\CF$ is the set of functions from a set $X$ to $\CC$ and $\{f_n\}_{n\geq 0}$  is a sequence of functions from $\CF$, we say $\{f_n\}_{n\geq 0}$ {\em converges pointwise} to $f \in \CF$ if for all $x \in X$
$$\lim_{n\rightarrow \infty}f_{n}(x) = f(x).$$
\end{definition}

\begin{definition}[Closure under pointwise convergence]
Let $F \subseteq \CF$ be a set of functions from $X$ to $\CC$. Then we say $F$ is closed under the topology of pointwise convergence if for every sequence $\{ f_n \}_{n \geq 0}$ in $\mathcal{F}$ converges pointwise to a function $f \in F$.
\end{definition}
We illustrate the notion of closure under pointwise convergence by an example where it does not hold. 
\begin{example} Let $\CF$ be the set of functions from $\RR$ to $\RR$ and let $F$ be the set of functions from $[0,1]$ to $[0,1)$. Consider the sequence $\{f_n\}_{n\geq 0}$ such that 
\[f_{i}(x) = x^{2} - \frac{1}{i}.\]
Clearly this sequence converges pointwise to $f(x) = x^2$. But since $f(1) = $, $f \notin F$ although every function in the sequence is in $F$. Hence, $F$ is not closed under the topology of pointwise convergence.
\end{example}

\begin{definition}[Gram matrix]
Suppose $X$ is a non-empty set and $\varphi: X \times X \rightarrow \CC$ is a kernel defined on $X$. Given $V = \{v_1,v_2,\ldots,v_n\} \subseteq X$ we will refer matrix $G$ with entries $G_{ij} = \varphi(v_i,v_j)$ as the {\em Gram matrix} of $V$ w.r.t $\varphi$.
\end{definition}
Note that this is a somewhat more general use of the term Gram matrix than is usually encountered in the linear algebra literature.

We now proceed to the properties of combinations of kernels. 
\begin{proposition}[Convex combination]
\label{prp:convexcone}
 P.d. (c.n.d) kernels  form a convex cone that is closed in the topology of pointwise convergence.
\end{proposition}
\begin{proof}
If $\varphi_{1}$ and $\varphi_{2}$ are $p.d.$ kernels defined on $X \times X$ and we consider $\{x_1, \ldots, x_n\} \subseteq X$ and $\{c_1, \ldots, c_n\} \subseteq \CC$ for some $n \in \NN$, clearly if $\alpha, \beta \geq 0$ then
\[\sum_{j,k}c_{j}\bar{c_{k}}\left(\alpha \varphi_{1}(x_{j},x_{k}) + \beta \varphi_{2}(x_{j},x_{k})\right) = \alpha \sum_{j,k}c_{j}\bar{c_{k}}\varphi_{1}(x_{j} x_{k}) + \beta \sum_{j,k}c_{j}\bar{c_{k}}\varphi_{2}(x_{j} x_{k}) \geq 0.\]
A similar argument shows that c.n.d. kernels form a convex cone.

Now consider a sequence of p.d. kernels $\{ \varphi_{n} \}_{n \geq 0}$. Choose $\{c_{1},\ldots,c_{n}\} \subseteq \mathbb{C}$ and $\{x_{1},\ldots,x_{n}\} \subseteq X$ for some $n \in \NN$ and consider 
$$ 
\lim_{n \rightarrow \infty} a_n =  \sum_{j,k} c_{j}\bar{c_{k}} \varphi_n(x_j,x_k).
$$
Since each $a_n \geq 0$, this limit, whenever it exists, will also be non-negative. Hence we have shown that the set of p.d. kernels is closed under pointwise convergence. A similar argument can be made for the set of c.n.d. kernels.
\end{proof}

\noindent{\bf Remark.} Closure under pointwise convergence does not hold for {\em strictly} p.d. (c.n.d.) kernels. For the p.d. case we can see this above since although each $a_n$ above may be strictly positive, the limit could be 0.

\begin{theorem}[Schur's theorem for products of p.d. kernels]
\label{thm:product-pd} 
If $\varphi_{1}, \varphi_{2}: X \times X \rightarrow \mathbb{C}$ are p.d. kernels then $\varphi_{1}\varphi_{2}: X \times X \rightarrow \mathbb{C}$ is also p.d.
\end{theorem}
\begin{proof}
All we need to show is that if $A = (a_{jk})$ and $B = (b_{jk})$ are p.d. then $C = (a_{jk}b_{jk})$ is p.d. By the definition of p.d. kernels and by Proposition~\ref{prp:pd-hermitian} a Gram matrix derived from a p.d. kernel is Hermitian and positive definite. So if we can prove that $C$ is p.d. then we can apply this to the Gram matrices of a set $\{x_1, \ldots x_n\} \subseteq X$ w.r.t. $\varphi_{1}$ and $\varphi_{2}$ for some $n \in \NN$. This will prove that the Gram matrix of the set w.r.t. $\varphi_{1}\varphi_{2}$ is also p.d. and we are done. 

The Spectral theorem for hermitian matrices says that $A$ is diagonalizable. Additionally, it has non negative eigenvalues because it is derived from a p.d. kernel,i.e., there is a unitary matrix $Q$ and a diagonal matrix $\Lambda$ with non-negative entries such that 
\begin{align*}
    A = Q^{*}\Lambda Q = Q^{*} \Lambda^{1/2} \Lambda^{1/2}Q = (\Lambda^{1/2}Q)^{*}(\Lambda^{1/2}Q)
\end{align*}
Therefore, we have $n$ functions $f_{1},\ldots,f_{n}: \{1,\ldots,n\}\rightarrow \CC$ such that
\begin{equation*}
    a_{jk} = \sum_{p=1}^{n}f_{p}(j)\overline{f_{p}(k)}
\end{equation*}
Therefore, for any $\{c_1,\ldots,c_n\} \in \CC$,
\begin{equation*}
    \sum_{j,k}c_{j}\overline{c_{k}}a_{jk}b_{jk} = \sum_{p=1}^{n}\sum_{j,k=1}^{n}c_{j}f_{p}(j)\overline{c_{k}f_{p}(k)} b_{jk}  
\end{equation*}
Let us analyse the inner summation on the RHS for any value of $p$. Put $\widetilde{c_{j}} = c_{j}f_{p}(j)$ and $\widetilde{c_{k}} = c_{k}f_{p}(k)$, then the inner summation simplifies to
$ \sum_{j,k}\widetilde{c_{j}} \overline{\widetilde{c_{k}}}b_{jk}$ which is non-negative because $B$ is p.d. 
Summing up over non-negative terms for all $p \in \{1,\ldots,n\}$, we have
$$ \sum_{j,k}c_{j}\overline{c_{k}}a_{jk}b_{jk} \geq 0$$
Hence, $C$ is p.d.
\end{proof}
Let us consider first a simple application of Theorem~\ref{thm:product-pd}.
\begin{proposition}
\label{prp:pd-real}
If $\varphi ; X \times X \rightarrow \CC$ is a p.d. kernel show that $\overline{\varphi}$, $\re \varphi$, and $|\varphi|^2$ are also p.d.
\end{proposition}
\begin{proof}
Since $\varphi$ is p.d., it follows that for $n \in \NN$, $\{x_1, \ldots, x_n\} \subseteq X$ and $\{c_1, \ldots, c_n\} \subseteq \CC$,
$ \sum_{j,k}c_{j}\overline{c_{k}}\varphi(x_{j},x_{k})\geq 0$.
Taking the complex conjugate of this, we have
$$ 0 \leq \overline{\sum_{j,k}c_{j}c_{k}\varphi(x_{j},x_{k})} = \sum_{j,k}\overline{c_{j}}c_{k}\overline{\varphi}(x_{j},x_{k}),$$
which proves that $\overline{\varphi}$ is p.d. Now, from Proposition~\ref{prp:convexcone} it follows that $\re \varphi = (\varphi + \overline{\varphi})/2$ is also p.d. From Theorem~\ref{thm:product-pd} it follows that $|\varphi|^2 = \varphi \overline{\varphi}$ is also p.d.
\end{proof}
\noindent{\bf Remark.} In general $|\varphi|$ is not guaranteed to be p.d.

Theorem~\ref{thm:product-pd} has some important consquences, two of which we discuss next.
\begin{corollary}[Tensor products of p.d. kernels are p.d.]
Let $\varphi_1 : X \times X \rightarrow \mathbb{C}$ and $\varphi_2 : Y \times Y \rightarrow \mathbb{C}$ be p.d. kernels. Then their tensor product $\varphi_1 \otimes \varphi_2 : (X \times Y) \times (X \times Y) \rightarrow \mathbb{C}$ defined by $\varphi_1 \otimes \varphi_2(x_1,y_1,x_2,y_2) = \varphi_1(x_1,x_2) \varphi_2(y_1,y_2)$ is also p.d.
\end{corollary}
\begin{proof}
Since $\varphi_1$ and $\varphi_2$ are defined on different spaces we cannot directly apply Theorem~\ref{thm:product-pd}, so we define $\widetilde\varphi_1$ and $\widetilde\varphi_2$ that transfer them over to the product space, i.e., $\widetilde\varphi_1, \widetilde\varphi_2: (X \times Y) \times (X \times Y) \rightarrow \mathbb{C}$. We define these as follows:
\[\widetilde\varphi_1(x_1,y_1,x_2,y_2) = \varphi_1(x_1,x_2)\]
$$\widetilde\varphi_2(x_1,y_1,x_2,y_2) = \varphi_2(y_1,y_2)$$
We know that $\widetilde\varphi_1$ and $\widetilde\varphi_2$ are p.d. kernels because $\varphi_1$ and $\varphi_2$ are p.d. By Theorem~\ref{thm:product-pd}, their product is also pd, and thus the tensor product, $\varphi_1 \otimes \varphi_2$ is p.d.
\end{proof}
\begin{corollary}
\label{cor:holomorphic}
Let $\varphi: X \times X \rightarrow \mathbb{C}$ be p.d. such that $|\varphi(x,y)| < \rho$, $\forall(x,y) \in X \times X$. Then if $f(z)=\sum_{n=0}^{\infty}a_{n}z^{n}$ is holomorphic in $\{z \in \CC: |z|<\rho \}$ and $a_{n} \geq 0$ for all $n \geq 0$, the composed kernel $f \circ \varphi$ is again p.d. In particular, if $\varphi$ is p.d. then so is $\exp(\varphi)$
\end{corollary}
\begin{proof}
Define $f_{n} = \sum_{i=0}^{n}a_{n}\varphi^{n}$.
By Theorem~\ref{thm:product-pd}, $\varphi^{n}$ is pd for all $n \geq 0$. Since $a_i \geq 0$ for all $i \geq 0$,  $f_{n}$ is a convex combination of p.d. kernels and is therefore, by  Proposition~\ref{prp:convexcone}, itself p.d.
$f(z)$  is holomorphic in $\{z \in \mathbb{C}: |z|<\rho \}$. Therefore, if $|\varphi(x,y)| < \rho$ for all $x,y \in X$, $f_{n}$ will converge pointwise to $\sum_{n=0}^\infty a_n \varphi^n$  which is guaranteed to be a p.d. kernel by Proposition~\ref{prp:convexcone} 
In particular $\exp(\varphi)$ is p.d. for any p.d. kernel $\varphi$ because the radius of convergence of $\sum_{n=o}^\infty z^n/n!$ is $\rho = \infty$.
\end{proof}

\subsubsection{A Sylvester-like criterion for kernels}

For strictly positive definite matrices, the following result relates the positive definite property to the value of the determininants of their principal submatrices. 

\begin{theorem}[Sylvester's criterion for strictly p.d. matrices]
\label{thm:matrix-sylvester}
An $n \times n$ Hermitian matrix $A = (a_{jk}: j,k\leq n)$ if strictly p.d. iff $\det(a_{jk} : j,k \leq p) > 0$ for $1\leq p \leq n$.
\end{theorem}
\begin{proof}
First, let's assume that $A$ is strictly p.d.
Therefore, by the Spectral Theorem for p.d. Hermitian matrices, $A$ can be diagonalized as $Q^* \Lambda Q$ where $Q$ is a unitary matrix and $\Lambda$ is a diagonal matrix. Since all the entries of $\Lambda$ are non-negative (in fact strictly positive), we can write $A = B*B$ where $B = Q\Lambda^{1/2}$.
Therefore, 
$\det(A) = |\det(B)|^2 \geq 0$. Further  $\det(A) \ne 0$, since other $\det(A - 0.I) = 0$ i.e. 0 is an eigen value for $A$, which contradicts the assumption that $A$ is strictly p.d. Hence, $\det(A) > 0$.

Since $A$ being strictly p.d. implies that $(a_{jk}: j,k \leq p)$ is strictly p.d. for every $p$ s.t. $1 \leq p \leq n$, we can use a similar argument to prove that $\det(a_{jk}: j,k \leq p) > 0$ for $1 \leq p \leq n$.

Now, we assume that $\det(a_{jk}: j,k \leq p) > 0$ for $1 \leq p \leq n$ we will show that $A$ is strictly p.d. The proof is by induction. Clearly this is true for $p = 1$ which is the base case for the induction. As our induction hypothesis we assume that $\det(a_{jk}: j,k \leq n-1) > 0$ implies that $(a_{jk}: j,k \leq n-1)$ is strictly p.d.

Now, let us consider the matrix $A = (a_{jk}: j,k \leq n-1)$ and transform it to $A' = (a'_{jk}: j,k \leq n-1)$ such that
\[a'_{jk} = a_{jk} - \left(\frac{a_{1k}}{a_{11}}\right)a_{j1},\]
i.e., we multiply the first column by $a_{1k}/a_{11}$ and subtract it from the $k$th column, for $2 \leq k \leq n$. This gives us
\[
A' = \begin{bmatrix} 
           a_{11} & 0 & \dots & 0  \\
           a_{21} & a'_{22} & \dots & a'_{2n}  \\
           \vdots & \vdots & \ddots & \vdots \\
           a_{n1} & a'_{n2} & \dots & a'_{nn} 
           \end{bmatrix}
\]
Since $A'$ is derived from $A$ through elementary column transformations, $\det(a_{jk}: j,k \leq p) = \det(a'_{jk}: j,k\leq p)$, $2 \leq p \leq n$. Further, $\det(a'_{jk}: j,k\leq p) = \det(b_{jk}: j,k\leq p)$ where $B = (b_{jk}: j,k\leq n)$ is defined as
\[
B = \begin{bmatrix} 
           a_{11} & 0 & \dots & 0  \\
           0 & a'_{22} & \dots & a'_{2n}  \\
           \vdots & \vdots & \ddots & \vdots \\
           0 & a'_{n2} & \dots & a'_{nn} 
           \end{bmatrix}
\]
Let $C = (c_{jk}: j,k \leq n-1)$ be defined as follows.
\[
C = \begin{bmatrix} 
    a'_{22} &  \dots & a'_{2n}  \\
    \vdots &  \ddots & \vdots \\
    a'_{n2} &  \dots & a'_{nn} 
    \end{bmatrix}
\]
Clearly $\det(c_{jk}: j,k \leq p) = 1/a_{11} \cdot \det(b_{jk}: j,k \leq p+1)$, $1 \leq p \leq n-1$. Since $a_{11} > 0$ and $\det(b_{jk} : j,k \leq p+1) = \det(a_{jk} : j,k \leq p+1) >0$ for all $1 \leq p \leq n-1$ we know that $\det(c_{jk} : j,k \leq p) > 0$ for $1 \leq p \leq n-1$, i.e.,  we have shown that all the principal submatrices of $C$ have strictly positive. We are now in a position to apply the induction hypothesis on $C$ to establish that it is strictly p.d. provided it is Hermitian. Consider $a'_{jk}$ such that $j,k \in \{2, \dots, n\}$. We have,
\[
\overline{a'_{jk}} = \overline{a_{jk} - \left(\dfrac{a_{1k}}{a_{11}}\right)a_{j1}}
= \overline{a_{jk}} - \left(\dfrac{\overline{a_{1k}}}{\overline{a_{11}}}\right)\overline{a_{j1}}
= a_{kj} - \left(\dfrac{a_{k1}}{a_{11}}\right)a_{1j}= a'_{kj}.
\]
Hence $C$ is Hermitian and so, by the induction hypothesis, it is strictly p.d.

With this in hand we will attempt to show that $A$ itself if strictly p.d. Consider an arbitrary non-zero vector $c = (c_1, c_2, \dots, c_n) \in \mathbb{C}^{n}$
\begin{align*}
\sum_{j,k=1}^{n}c_{j}\overline{c_{k}}a_{jk} 
&= \sum_{j,k=2}^{n}c_{j}\overline{c_{k}}\left(a'_{jk} + \dfrac{a_{1k}a_{j1}}{a_{11}}\right) + \sum_{j=2}^{n}c_{j}\overline{c_{1}}a_{j1} + \sum_{k=2}^{n}c_{1}\overline{c_{k}}a_{1k} + |c_{1}|^{2}a_{11}\\
&= \sum_{j,k=2}^{n}c_{j}\overline{c_{k}}a'_{jk} + \dfrac{1}{a_{11}}\left(|\sum_{j=2}^{n}c_{j}a_{j1}|^{2} + 2a_{11}Re\left(c_{1}\sum_{k=2}^{n}\overline{c_{k}}a_{1k}\right) + (|c_{1}|a_{11})^{2}\right)\\
&= \sum_{j,k=2}^{n}c_{j}\overline{c_{k}}a'_{jk} + \dfrac{1}{a_{11}}|\sum_{j=1}^{n}c_{j}a_{j1}|^{2}
\end{align*}
Now if there is $i \in \{2, \dots, n\}$  such that $c_{i} \neq 0$ then the first term is positive and second term is non-negative. On the other hand if  $c_{i} = 0$ for all $i \in \{2, \dots, n\}$ then the first term is zero and the second term is positive since $c_{1} \neq 0$.
\end{proof}
\noindent {\bf Remark:} Sylvester's criterion does not hold for p.d. matrices that are not strictly p.d. For example, consider 
    \begin{align*}
        A = \begin{pmatrix} 
    0 & 0 \\
    0 & -1
    \end{pmatrix}
    \end{align*}
Here, $A$ is a Hermitian $2 \times 2$ matrix such that the determinants of both principal submatrices are 0 but $A$ is not p.d.

We now turn to establishing a Sylvester-like criterion for p.d. kernel.

\begin{theorem}
\label{thm:kernel-sylvester}
If $\varphi: X \times X \rightarrow \mathbb{C}$ is a kernel then $\varphi$ is p.d. iff 
\[\det(\varphi(x_{j},x_{k}) : j,k \leq n) \geq 0\]
for all $n \in \mathbb{N} \text{ and } x_{1}, \dots, x_{n} \in X$
\end{theorem}
\begin{proof}
If $\varphi$ is p.d. then, by Proposition~\ref{prp:pd-hermitian}, the $n \times n$ matrix given by $A = (\varphi(x_{j},x_{k}) : j,k\leq n)$ is a p.d. Hermitian matrix and we can apply an argument similar to the one made in the proof of Theorem~\ref{thm:matrix-sylvester} to prove that all its principal submatrices have non-negative determinant.

So, let us turn to the other direction, i.e., let $\varphi: X \times X \rightarrow \mathbb{C}$ be a kernel such that for any $n \in \NN$ and $x_{1}, \dots, x_{n} \in X$, $\det(\varphi(x_{j},x_{k}) : j,k \leq n) \geq 0$.  Define $\varphi_{\varepsilon} = \varphi + \varepsilon I_{\Delta}$ where $\varepsilon > 0$ and $\Delta$ is the diagonal in $X \times X$, i.e., $\varphi_{\varepsilon}$ adds a small positive constant to $\varphi(x,x)$ for each $x \in X$.

Computing the determinant of the matrix created by modified version of $\varphi$, and assuming that $x_1, \ldots, x_n$ are distinct, we see that 
$$\det(\varphi_{\varepsilon}(x_{j},x_{k}):j,k \leq n) = \sum_{p=0}^{n}d_{p}\varepsilon^{p},$$
where $d_{n} = 1$ and 
$$d_{p} = \sum_{\substack{A \subseteq \{1, \dots, n\} \\ |A| = (n-p)}} \det(\varphi(x_{j},x_{k}): j,k \in A)$$
for $0\leq p \leq (n-1)$. Since each of the terms in the sum defining $d_p$ is non-negative, therefore $d_p$ is non-negative for $0\leq p \leq (n-1)$. This implies that $\det(\varphi_{\varepsilon}(x_{j},x_{k}):j,k \leq n) \geq \varepsilon^n > 0$. Since the same argument can be used to establish that $\det(\varphi_{\varepsilon}(x_{j},x_{k}):j,k \leq p) > 0$ for $1 \leq p \leq n-1$, we get by Theorem~\ref{thm:matrix-sylvester} that $\varphi_{\varepsilon}$ is a p.d. kernel. 

Clearly $\varphi$ is the pointwise limit $\varepsilon \downarrow 0$, and since the convex cone of p.d. kernels is closed under the topology of pointwise convergence (Proposition~\ref{prp:convexcone}) we see that $\varphi$ is p.d.
\end{proof}

\subsection{Relating p.d. and c.n.d. kernels}
We now discuss the relationships between p.d. and c.n.d. kernels. Clearly if $\varphi$ is p.d. then $-\varphi$ is c.n.d. Although the converse need not be true, the following lemma gives a useful relationship in the opposite direction. 
\begin{lemma}
\label{lem:cnd-to-pd}
Let $X$ be a non-empty set, $x_{0} \in X$ and $\psi: X \times X \rightarrow \CC$ be a Hermitian kernel. Let $\varphi: X \times X \rightarrow \mathbb{C}$ be such that $\varphi(x,y) = \psi(x,x_{0}) + \overline{\psi(y,x_{0})} - \psi(x,y) - \psi(x_{0},y_{0})$. Then $\varphi$ is p.d. iff $\psi$ is c.n.d. Further, if $\psi(x_{0},x_{0}) \geq 0$ and $\varphi_{0}(x,y) = \psi(x,x_{0}) + \overline{\psi(y,x_{0})} - \psi(x,y)$ then $\varphi_{0}$ is p.d. iff $\psi$ is c.n.d.
\end{lemma}
\begin{proof}
\noindent{(i)}  Let $\varphi$ be a p.d. kernel. For $n \geq 2$ consider $c_{1}, \ldots, c_{n} \in \CC$ such that $\sum_{i=1}^{n}c_{i} = 0$ and $x_{1}, x_{2}, \dots, x_{n} \in X$. Since $\varphi$ is p.d. we have that 
$$\sum_{j,k=1}^{n}c_{j}\overline{c_{k}}(\psi(x_{j},x_{0}) + \overline{\psi(x_{k},x_{0})} - \psi(x_{j},x_{k}) - \psi(x_{0},y_{0}))  \geq 0.$$
Since $\sum_{i=1}^{n}c_{i} = 0$, three of the terms on the LHS become 0 and we are left with 
$$- \sum_{j,k=1}^{n} \psi(x_{j},x_{k})  \geq 0,$$
which proves that $\psi$ is c.n.d. By a similar argument we can show that $\psi$ is c.n.d. if $\varphi_0$ is p.d.

\noindent{(ii)} Let $\psi$ be a c.n.d. kernel. For some $n \in \NN$ consider  $c_{1}, c_{2}, \dots, c_{n} \in \CC$ and  $x_{1}, x_{2}, \dots, x_{n} \in X$. Choose $x_{0} \in X$ and set $c_{0} = -\sum_{i=1}^{n}c_{i}$. Since $\psi$ is c.n.d. we know that  
$$\sum_{j,k=0}^{n}c_{j}\overline{c_{k}}\psi(x_{j},x_{k})  \leq 0.$$
We split the sum on the LHS, separating out the terms involving $c_0$ and $x_0$. This gives us that 
$$\sum_{j,k=1}^{n}c_{j}\overline{c_{k}}\psi(x_{j},x_{k}) +  \sum_{j=1}^{n}c_{j}\overline{c_{0}}\psi(x_{j},x_{0}) +  \sum_{k=1}^{n}c_{0}\overline{c_{k}}\psi(x_{0},x_{k}) +
|c_{0}|^{2}\psi(x_{0},x_{0}) \leq 0.$$
Replacing $c_{0}$ with $ -\sum_{i=1}^{n}c_{i}$ on the LHS we get
\begin{equation}
\label{eq:cnd-to-pd}
\sum_{j,k=1}^{n}c_{j}\overline{c_{k}}(\psi(x_{j},x_{k}) - \psi(x_{j},x_{0}) - \psi(x_{0},x_{k}) + \psi(x_{0},y_{0}))  \leq 0,
\end{equation}
i.e.,
$$-\sum_{j,k=1}^{n}c_{j}\overline{c_{k}}\varphi(x_{j},x_{k})  \leq 0.$$
So, $\varphi$ is p.d. Note that if $\psi(x_{0},y_{0}) \geq 0$ then by~\eqref{eq:cnd-to-pd} we have,
$$\sum_{j,k=1}^{n}c_{j}\overline{c_{k}}(\psi(x_{j},x_{k}) - \psi(x_{j},x_{0}) - \psi(x_{0},x_{k}))  \leq 0,$$
i.e., $\varphi_0$ is p.d.
\end{proof}

The form of Lemma~\ref{lem:cnd-to-pd} doesn't appear so easy to work with but it has some nice consequences. We present one such here, which is attributed ot Schonberg in~\cite{berg:1984}.

\begin{theorem}
\label{thm:schonberg}
If $\psi: X \times X \rightarrow \CC$ is a kernel on a non-empty set $X$ then $\psi$ is c.n.d. iff $\exp(-t\psi)$ is a p.d. kernel $t > 0$.
\end{theorem}
\begin{proof}
\noindent (i) First assume that for $t > 0$, $exp(-t\psi)$ is a p.d. kernel. Therefore, $(1 - exp(-t\psi))$ is a c.n.d. kernel and so is $\left(\dfrac{1 - exp(-t\psi)}{t}\right)$. Since
$$ \psi = \lim_{t \rightarrow 0} \dfrac{1}{t}(1 - exp(-t\psi)), $$
by the fact that the convex cone of c.n.d. kernels is closed under pointwise convergence (Proposition~\ref{prp:convexcone}) we infer that $\psi$ is a c.n.d. kernel.

\noindent (ii) Let us assume that $\psi$ is a c.n.d. kernel. It is sufficient to show that $\exp(\psi)$ is p.d. as we can replace $\psi$ by $t\psi$ because if $\psi$ is c.n.d. then $t\psi$ is c.n.d. whenever  $t > 0$.
Let us choose $x_{0} \in X$ and $\varphi$ as in Lemma~\ref{lem:cnd-to-pd} to get,
$$-\psi(x,y) = \varphi(x,y) - \psi(x,x_{0}) - \overline{\psi(y,x_{0})} + \psi(x_{0},x_{0}).$$
Taking exponents on both sides we have
$$\exp(-\psi(x,y)) = \exp(\varphi(x,y))\cdot\exp(\psi(x,x_{0}))\cdot\exp(\overline{\psi(y,x_{0})})\cdot \exp(\psi(x_{0},x_{0}))$$
Examining the RHS we see that it is a product of four terms. The first of these is the exponent of a p.d. kernel and therefore is p.d. by Corollary~\ref{cor:holomorphic}. The product $\exp(\psi(x,x_{0}))\cdot exp(\overline{\psi(y,x_{0})})$ is a p.d. kernel by Proposition~\ref{prp:univariate}. And $\exp(\psi(x_{0},x_{0}))$ is a p.d. kernel trivially as it is a positive constant. By Schur's theorem (Theorem~\ref{thm:product-pd}), the product of these three p.d. kernels is also a p.d. kernel.
\end{proof}

\noindent{\bf Remarks.}
\begin{enumerate}
\item Applying Theorem~\ref{thm:schonberg} to Example~\ref{ex:gaussian-cnd} extended to $d$-dimensions tells us that the Gaussian Radial Basis Function, $\varphi(x,y) = e^{-c\|x-y\|^2}$, a very popular kernel used widely in Machine Learning applications, is indeed p.d.
\item It is also possible to show that $\psi$ is c.n.d. iff $1/(t + \psi)$ is p.d. for all $t > 0$ but we omit the proof here, refering the reader to~\cite{berg:1984} for this and other interesting relationships between p.d.  and c.n.d. kernels.
\end{enumerate}

\section{Hilbert spaces and kernels}
\label{sec:hilbert}

In this section we will establish the relationship the role of Hilbert spaces in the study of p.d. kernels. We will discuss Mercer's theorem and also present the Reproducing Kernel Hilbert Space associated with a p.d. kernel.

\subsection{Inner products and their associated norms}
\label{sec:hilbert:hilbert}

To make these notes self-contained we first present some basic definitions and results to develop the definition of a Hilbert space. The primary source of the material in this section is the chapter by Heil~\cite{heil}.

\begin{definition}[Semi-Inner Product]
Let $X$ be a vector space over $\CC$. A function $\langle\cdot, \cdot\rangle : X \times X \rightarrow \CC$ is called a {\em semi-inner product} if: 
\begin{enumerate}
\item $\langle x, x\rangle \geq 0$ for all $x \in X$  ,
\item (Hermitian) $\langle x, y\rangle=\overline{\langle y, x\rangle}$ $\forall$ $x, y \in X$ and
\item (Linearity in the first variable) $\langle\alpha x+\beta y, z\rangle=\alpha\langle x, z\rangle+\beta\langle y, z\rangle$, for all $x, y, z \in X$ and $\alpha, \beta \in \CC$. 
\end{enumerate}
\end{definition}


\noindent{\bf Remarks.}
\begin{itemize} 
\item (Anti-linearity in the second variable). From the Hermitian property and linearity in the first variable it is easy to see that for all $x, y, z \in X$ and $\alpha$, $\beta \in \CC$  
\[\langle x, \alpha y+\beta z\rangle=\bar{\alpha}\langle x, y\rangle+\bar{\beta}\langle x, z\rangle.\]
\item It is also easy to deduce that $\langle x, 0 \rangle = 0 = \langle 0, x \rangle$ for all $x \in X$. Further, $\langle 0,0\rangle = 0$.
\end{itemize}

Semi inner-products have a property of interest to us: they are p.d. kernels. 
\begin{proposition}
\label{prp:sip-pd}
A semi-inner product $\langle\cdot, \cdot\rangle: X \times X \rightarrow \CC$  is a p.d. kernel.
\end{proposition}
\begin{proof} 
Let $n \in N$, $\{x_{1},\ldots,x_{n}\} \subseteq X$ and $\{c_{1},\ldots,c_{n}\} \subseteq \mathbb{C}$. 
Since a semi-inner product is linear in the first variable and anti-linear in the second variable we get,
 $$\sum_{j,k=1}^nc_{j}\bar{c_{k}}\langle x_{j},x_{k} \rangle = \sum_{j,k=1}^n\langle c_{j} x_{j},c_{k}x_{k}\rangle = \left\langle \sum_{j=1}^n c_{j}x_{j},\sum_{j=1}^nc_{j}x_{j}\right\rangle \geq 0.$$
\end{proof}

\begin{definition}[Inner Product]
If a semi-inner product space $\langle\cdot, \cdot\rangle$ satisfies
\[\langle x, x\rangle= 0 \Longrightarrow x=0\]
then it is called an {em inner product}. 
\end{definition}

\begin{definition}[Inner product space]
A vector space $X$ with inner product $\langle \cdot,\cdot \rangle$ defined on it is called an {\em inner product space} or a {\em  pre-Hilbert space.}
\end{definition}
We now discuss some examples of inner products. 

\begin{example}
\label{ex:dot-prod}
Prove that the dot product $\bm{x} \cdot \bm{y}=\bm{x}_{1} \bar{\bm{y}}_{1}+\cdots+\bm{x}_{n} \bar{\bm{y}}_{n}$ is an inner product on $\mathbb{C}^{n}$.
\end{example}
\begin{proof}
First we see that 
\[
        \bm{x} \cdot \bm{x}=\bm{x}_{1} \bar{\bm{x}}_{1}+\cdots+\bm{x}_{n} \bar{\bm{x}}_{n} = \sum_{k}|\bm{x}_{k}|^{2} \geq 0. 
\]
This further implies $\bm{x} \cdot \bm{x}$ is 0 only if $\bm{x}  \bm{0}$.

For the Hermitian property note that 
    \begin{equation*}
         \overline{\bm{y} \cdot \bm{x}}= \overline{\sum_{i=1}^n \bm{y}_i \overline{\bm{x}_i}}
 =  \sum_{i=1}^n \bm{x}_i \overline{\bm{y}_i} =  \bm{x} \cdot \bm{y}.
    \end{equation*}
Linearity in the first variable is easy to see.
\end{proof}

\begin{example}
\label{ex:weighted-dot-prod}
 If $w_{1}, \ldots, w_{n} \geq 0$ are fixed scalars, then the weighted dot product $\langle x, y\rangle= x_{1} \bar{y}_{1} w_{1}+$ $\cdot+x_{n} \bar{y}_{n} w_{n}$ is a semi-inner product on $\mathbb{C}^{n}$. It is an inner product if $w_{i}>0$ for each $i .$
\end{example}
This example can be established just like Example~\ref{ex:dot-prod}, except here we note that even if there is a $w_i = 0$ then $\langle \bm{x}, \bm{x}\rangle$ is 0 for any $\bm{x}$ which is non-zero only in the $i$th coordinate. Hence $\langle \cdot, \cdot \rangle$ is only a semi-inner product unless $w_i > 0$ for all $i$. 

\begin{example}
\label{ex:matrix-ip}
Let $A$ be a $n \times n$ p.d. matrix.  $\langle \bm{x},\bm{y} \rangle = A\bm{x} \cdot \bm{y} $ is a semi-inner product on $\mathbb{C}^n$ where $\cdot$ is as defined in Example~\ref{ex:dot-prod}. Moreover, $\langle \cdot, \cdot\rangle$ is an inner product if $A$ is strictly p.d.
\end{example}
\begin{proof}
Clearly $\langle \bm{x}, \bm{x}\rangle$ is non-negative if $A$ is p.d. and positive for all $\bm{x} \ne \bm{0}$ if $A$ is strictly p.d. Linearity in the first variable follows easily since $A$ is a linear transformation. It remains to show the Hermitian property.

\[   \overline{A\bm{y} \cdot x} =\overline{\sum_{i=1}^{n} \left(\sum_{j=1}^{n} A_{ij}\bm{y}_j\right)\overline{x_{i}}} = = \sum_{i, j = 1}^{n} \overline{A_{ij}\bm{y}_j}x_i.\]
Since all p.d. matrices are Hermitian (by arguments similar to that in the proof of Proposition~\ref{prp:pd-hermitian}), the last term above can be rewritten as \[\sum_{i, j = 1}^{n} A_{ji}\overline{\bm{y}_j}x_i ={\sum_{j=1}^{n} \left(\sum_{i=1}^{n} A_{ji}x_i\right){\bm{y}_{j}}} = A\bm{x}\cdot\bm{y}.\]
\end{proof}
The following example shows that all semi-inner products on $\CC^n$ can be represented by a p.d. matrix. 
\begin{example}
For an arbitrary semi-inner product $\langle \cdot, \cdot \rangle$ on $\mathbb{C}^{n}$, there is  a p.d matrix A such that  $\langle \bm{x}, \bm{y} \rangle = A\bm{x}\cdot\bm{y}$. 
\end{example}
\begin{proof}
Let $\bm{e}^1, \ldots, \bm{e}^n$, where $\bm{e}^i$ is the vector that is 1 in the $i$th coordinate and 0 elsewhere, be the standard basis of $\CC$. We hypothesise that the $A$ is the Gram matrix of this basis w.r.t the given inner product, i.e., $A_{ij} = \langle \bm{e}^i,\bm{e}^j\rangle$, $1 \leq i,j\leq n$. 
To verify this note, using the properties that any semi-inner product is linear in the first variable and antilinear in the second variable, that 
\[    \langle \bm{x}, \bm{y} \rangle =  \left\langle \sum_{i=1}^n \bm{x}_i \bm{e}^i,\sum_{i=1}^n \bm{y}_{i}\bm{e}^{i} \right\rangle = \sum_{i,j=1}^n \bm{x}_i \overline{\bm{y}_j} \langle \bm{e}^i, \bm{e}^j \rangle. 
\]
It is easy to see that $A$ is p.d. since $\bm{x}^TA\bm{x} = \langle \bm{x}, \bm{x}\rangle$ which is non-negative since $\langle \cdot, \cdot \rangle$ is an semi-inner product.
\end{proof}

We now discuss some basic properties of inner products. With every semi-inner product $\langle\cdot,\cdot\rangle$ we will associate a univariate function: $\|x\| = \langle x, x\rangle^{1/2}$. 

\begin{lemma}[Polar Identity]\label{lem:polar} If $\langle \cdot,\cdot\rangle$ is a semi-inner product on a vector space $X$ then for all $x,y \in X$
\[
\|x+y\|^{2} = \|x\|^{2} + 2\re\langle x,y \rangle+ \|y\|^{2}.
\]
\end{lemma}
\begin{proof}
\begin{equation*}
    \|x+y\|^{2} = \langle x+y,x+y\rangle = \|x\|^{2} + \|y\|^{2}\ + \langle x,y \rangle + \langle y,x \rangle
\end{equation*}
The result follows by using the Hermitian property,  $\langle x,y \rangle = \overline{\langle y,x \rangle}$.
\end{proof}

\begin{theorem}[Cauchy-Bunyakovsky-Schwarz Inequality]\label{thm:cbs} If $\langle\cdot, \cdot\rangle$ is a semi-inner product on a vector space $X$ then for all $x,y \in X$
\[
|\langle x, y\rangle| \leq\|x\|\|y\|.
\]
Moreover, equality holds if $x = \lambda y$ for some $\lambda \in \mathbb{C}$ 
\end{theorem}
\begin{proof}
If either of $x$ or $y$ are 0 the result follows trivially so we assume $x,y$ are both non-zero. Note that if $x \in X$ and $\alpha \in \CC$ then by the linearity and anti-linearity properties of semi-inner products $\|\alpha x\| = |\alpha|\|x\|$. With this in mind, by Lemma~\ref{lem:polar}, for any $t \in \CC$ we get that 
\begin{equation*}
    0 \leq \|x+ty\|^{2} \leq \|x\|^{2}+2 \re\langle x, ty \rangle+|t|^{2}\|y\|^{2}
\end{equation*}
Setting  $t=-\frac{\langle x, y\rangle}{\langle y,y \rangle}$ we get 
\[
\|x\|^{2}-2 \re \frac{|\langle x,y\rangle|^{2}}{\| y\|^{2}} +\frac{|\langle x, y\rangle|^{2}\|y\|^{2}}{\|y\|^{4}} = \|x\|^{2}-2 \frac{|\langle x,y\rangle|^{2}}{\| y\|^{2}} +\frac{|\langle x, y\rangle|^{2}}{\|y\|^{2}} = \|x\|^{2}- \frac{|\langle x,y\rangle|^{2}}{\| y\|^{2}} \geq 0,
\]
and the result follows. If $x = \lambda y$ for some $\lambda \in \CC$, then
\[
  \langle\lambda y, \lambda y\rangle-\frac{|\langle\lambda y, y\rangle|^{2}}{\|y\|^{2}}  
=  |\lambda|^{2}\|y\|^{2}-|\lambda|^{2} \frac{\|y\|^{4}}{\|y\|^{2}}=0
\]
\end{proof}

\noindent\textbf{Remark.} Both the Polar identity and the Cauchy-Bunyakovsky-Schwarz Inequality don't need $\langle \cdot, \cdot\rangle$ to be an inner product, semi-inner product is enough.

Now, let us see some two examples of inner products in spaces which are not finite dimensional.
\begin{example}[Square summable sequences]
Let $I$ be a countable index set (eg. set of natural numbers or integers). Let $w: I \rightarrow[0, \infty)$. Let 
\begin{equation*}
    \ell_{w}^{2}=\ell_{w}^{2}(I)=\left\{x=\left\{x_{i}\right\}_{i \in I}: \sum_{i \in I}\left|x_{i}\right|^{2} w(i)<\infty\right\}
\end{equation*}
then 
\begin{equation*}
    \langle x, y\rangle=\sum_{i \in I} x_{i} \bar{y}_{i} w(i) 
\end{equation*}
defines a semi-inner product on $\ell^2_w$. If $w(i)>0$ for all $i\in I$ then it is an inner product
\end{example}
\begin{proof}
The non-negativity of diagonal elements is easy to prove and, like in Example~\ref{ex:weighted-dot-prod} we can argue that $\langle x, x\rangle$ is 0 whenever $x = 0$ only if every $w_i > 0$. Linearity in the first variable is also easy to establish. But a complication arises with the Hermitian property. We know that
\[\overline{\langle x, y \rangle} = \overline{\sum_{i \in I}x_i\overline{y_i}w(i)} = \sum_{i \in I}\overline{x_i}y_iw(i),\]
but there is no guarantee that the last sum is finite. To see that it is, let $I = \NN$ and consider the partial sum of the first $n$ terms of the sum. By Example~\ref{ex:weighted-dot-prod} and by the Cauchy-Bunyakovski-Schwarz inequality (Theorem~\ref{thm:cbs}) we know that 
\[\sum_{i \in =1}^n \overline{x_i}y_iw(i) \leq \sum_{i=1}^n |x_i|^2 w_i \sum_{i=1}^n |y_i|^2 w_i.\]
Taking limits on both sides we get that the RHS is finite since $x,y \in \ell^2_w$, and hence the LHS converges to a finite quantity and hence is equal to $\langle y,x\rangle$.
\end{proof}

\begin{example}[Square integrable functions]
\label{ex:square-integ}
Let $(X, \Omega, \mu)$ be a measure space. Define
$$
L^{2}(X,\mu)=\left\{f: X \rightarrow \mathbb{F}: \int_{X}|f(x)|^{2} d \mu(x)<\infty\right\}
$$
where we identify functions that are equal almost everywhere, i.e.,
$$
f=g \Longleftrightarrow \mu\{x \in X: f(x) \neq g(x)\}=0.
$$
Then
$$
\langle f, g\rangle=\int_{X} f(x) \overline{g(x)} d \mu(x)
$$
defines an inner product on $L^{2}(X,\mu)$.
\end{example}

We now introduce the notion of a norm and show that an inner product naturally induces a norm.
\begin{definition}[Semi Norm]
Given a vector space $X$, a function $f : X  \rightarrow \mathbb{R}$ is said to be {\em semi norm}, if for all $x, y \in X$ and $\alpha \in \CC$
\begin{enumerate}
    \item $f(x) \geq 0$, 
    \item $f(\alpha x) = |\alpha| f(x),$ and
    \item $f(x+y) \leq  f(x) + f(y)$. (Triangle Inequality)
\end{enumerate}
Further, if $f(x) = 0 \implies x=0$, then $f$ is called a {\em norm}.
\end{definition}

\begin{definition}[Normed Linear Space]
Any vector space equipped with a norm is called as a {\em normed linear space} or a {\em normed vector space.} 
\end{definition}

\begin{proposition}
$\|\cdot \|$ is a semi-norm if $\langle \cdot,\cdot \rangle$ is a semi-inner product and a norm if $\langle \cdot,\cdot \rangle$ is an inner product.
\end{proposition}
\begin{proof}
The non-negativity of $\|\cdot \|$ follows easily. The second property follows from the fact that a semi-inner product is linear in the first variable and anti-linear in the second variable. We now argue for the Triangle inequality.

From the Polar identity (Lemma~\ref{lem:polar}) we know that
    \begin{equation*}
     \|x+y\|^{2}=\|x\|^{2}+2 \re \langle x, y \rangle+\|y\|^{2}.   
    \end{equation*}
Since $\re \alpha \leq |\alpha|$ for any $\alpha \in \CC$, we can say that 
\[     \|x+y\|^{2}\leq\|x\|^{2}+2 |\langle x, y \rangle|+\|y\|^{2}.   \]
Applying the Cauchy-Bunyakovsky-Schwarz Inequality (Theorem~\ref{thm:cbs}), we get 
\[ \|x+y\|^{2}\leq \|x\|^{2}+2 \|x\|\|y\|+\|y\|^{2} = (\|x\|+\|y\|)^{2}.\]
\end{proof}
Thus we see that an inner product induces a norm. Now we will see that just like the inner product is a p.d. kernel, the norm provides a natural c.n.d. kernel. The following proposition is a generalization of Example~\ref{ex:gaussian-cnd}.

\begin{proposition}
\label{prp:seminorm-cnd}
\(\psi(x, y) =\|x-y\|^{2}\) is a c.n.d. kernel if  $\|x\|$ is a semi norm. 
\end{proposition}
\begin{proof}
Given $n \geq 2,\left\{x_{1}, \ldots, x_{n}\right\} \subseteq X$ and $\left\{c_{1}, \ldots, c_{n}\right\} \subseteq \mathbb{C}$ such that \(\sum_{i=1}^{n} c_i = 0\) by the Polar Identity (Lemma~\ref{lem:polar}) we have that 

\begin{equation*}
    \sum_{j, k=1}^{n} c_{j} \overline{c_{k}} \|x_j-x_k\|^{2}    = \underbrace{\sum_{j, k=1}^{n} c_{j} \overline{c_{k}} \|x_j\|^{2}}_{0} - 2 \sum_{j, k=1}^{n} c_{j} \overline{c_{k}} \re\langle x_j, x_k\rangle+  \underbrace{\sum_{j, k=1}^{n} c_{j} \overline{c_{k}} \|x_k\|^{2}}_{0},
\end{equation*}
where the first and the third term are 0 because $\sum_{i=1}^n c_i =0$. From Proposition~\ref{prp:sip-pd} we know that $\langle \cdot, \cdot\rangle$ is a p.d. kernel and so, from Proposition~\ref{prp:pd-real}, we know that $\re \langle \cdot, \cdot\rangle$ is a p.d. kernel. Therefore
\begin{equation*}
    \sum_{j, k=1}^{n} c_{j} \overline{c_{k}} \|x_j-x_k\|^{2} =  - 2 \sum_{j, k=1}^{n} c_{j} \overline{c_{k}} \re\langle x, y\rangle \leq 0 
\end{equation*}
\end{proof}

\subsection{Hilbert spaces}

Hilbert spaces are inner product spaces with certain convergence properties. To make the definition concrete, we first need to define notions of convergence andinvestigate some properties of inner product spaces under these notions. Note that Definition~\ref{def:convergence} and Proposition~\ref{prp:conv-properties} hold for any linear space which has a norm associated with it, i.e., a more general class of spaces than inner product spaces. 

\begin{definition}[Convergence in a normed space]
\label{def:convergence} Let $X$ be a normed linear space and let $\{f_n\}^{\mathrm{\infty}}_{n=1}$ be a sequence of elements of $X$.
\begin{enumerate}
\item We say that $\{f_n\}$ converges to $f\in X$ , denoted $f_n \to f$, if 
\begin{equation*}
\lim_{n\to\infty}\|f-f_n\| = 0,
\end{equation*}
i.e.,
\begin{equation*}
\forall\varepsilon>0, \exists N>0 \text{ such that } n>N \implies \|{f-f_n}\| < \varepsilon.
\end{equation*}
\item We say that $\{f_n\}$ is Cauchy if 
\begin{equation*}
\forall \varepsilon>0 , \exists N>0 \text{ such that } m,n>N \implies\|{f_m-f_n}\|<\varepsilon.
\end{equation*}
\end{enumerate}
\end{definition}

\begin{proposition}[Convergence properties of normed linear spaces]\label{prp:conv-properties} Let $X$ be a normed linear space and $x, y \in X$,
\begin{enumerate}
\item Reverse Triangle Inequality: $\|{x-y}\| \ \geq \ \bigl| \|{x}\|-\|{y}\|\bigr|$
\item Continuity of the norm: $x_n \to x \implies \|{x_n}\| \to \|{x}\|$
\item Continuity of the inner product: \textit{If X is an inner product space then }
\begin{equation*}
x_n \to x, \ y_n \to y \implies \langle x_n, y_n\rangle \to \langle x,y\rangle
\end{equation*}
\item All convergent sequences are bounded, and the limit of a convergent sequence is unique.
\item Cauchy sequences are bounded.
\item Every convergent sequence is Cauchy. 
\item There exist inner product spaces for which not every Cauchy sequence is convergent.
\end{enumerate}
\end{proposition}
\begin{proof}
\begin{enumerate}
\item By the  triangule inequality we have
\begin{equation*}
\|{x}\| + \|{y-x}\| \ \geq \ \|{x+(y-x)}\|  = \|{y}\|
\end{equation*}
and
\begin{equation*}
\|{x-y}\| + \|{y}\| \ \geq \ \|{(x-y)+y}\| = \|{x}\|
\end{equation*}
When we move $\|{x}\|$ and $\|{y}\|$ to the RHS in the inequalities above, we have 
\begin{equation*}
\|{y-x}\| \ \geq \ \|{y}\| - \|{x}\|
\end{equation*}
and
\begin{equation*}
\|{x-y}\| \ \geq \ \|{x}\| - \|{y}\|
\end{equation*}
We know that $\|{y-x}\| = \|{x-y}\|$, therefore, we have $\|{x-y}\| \ \geq \ \bigl| \|{x}\|-\|{y}\|\bigr| $.
\item For some $\epsilon > 0$, if $\|x_n-x\| < \varepsilon$ then by the reverse triangle inequality above , we have
$\|{x_n}\| - \|{x}\| \ \leq \ \|{x_n-x} \| < \varepsilon.$
\item We can write
\begin{align*}
 \bigl| \langle x_n, y_n\rangle - \langle x,y \rangle \bigr| &= \bigl| \langle x_n, y_n \rangle -\langle x_n, y \rangle + \langle x_n, y \rangle + \langle x, y \rangle\bigr| \\
&\leq \bigl|\langle x_n, y_n \rangle -\langle x_n, y\rangle\bigr| + \bigl|\langle x_n, y \rangle -\langle x, y\rangle \bigr| \\
&\leq\|{x_n}\| \cdot \|{y_n- y}\| + \|{x_n-x} \| \cdot \|{y}\|  
\end{align*}
where the last step follows from the Cauchy-Bunyakovsky-Schwarz inequality (Theorem~\ref{thm:cbs}).
Since the norm is continuous, i.e., 
$
\|{x_n}\| \to \|{x}\| \text{ and } \|{y_n}\| \to \|{y}\|$, therefore, 
\begin{equation*}
\bigl| \langle x_n, y_n\rangle - \langle x,y \rangle \bigr| \to 0.
\end{equation*}
\item 
(i) Let $f_n$ be a convergent sequence and $f_n \to f$. If we take $\varepsilon  = 1$, we have an $N > 0$ such that $|f_n - f| < 1$ for all $n > N$. Using the reverse triangle inequality, for $n > N$, 
\begin{equation*}
| f_n| -| f| < 1 \implies | f_n| < | f| + 1
\end{equation*}
Let $M = \max\{| f | + 1,| f_1 |,\ldots, | f_N |  $\}, we have $| f_n| \leq M$ for all $n > N$.
\\
(ii) Uniqueness of the limit can be proved by assuming the contrary and showing a contradiction.
\item By the definition of Cauchy sequence for every $\varepsilon > 0$ there is an $ N>0$ such that for  $n,m > N,$ $| f_n - f_m| < \varepsilon$. Using the triangle inequality, we have 
\[| f_n| = | f_n - f_m + f_m|  \leq \ | f_n- f_m| + | f_m|. \]
Taking  $\varepsilon$ = 1, we have that $| f_n - f_m| < 1$ for $n, m > N$. 
Now taking $ m = N+1$, we get  $| f_n| < 1+| f_{N+1}|$
This bounds all the terms beyond the $N$th term. Taking $M = \max\{ | f_1 |,\ldots,| f_n |,| f_{N+1} | + 1  $\}, we have $| f_n| \leq M$ for all $n | 0$.
\item Let $f_n$ be a convergent sequence, we have $| f_n - f| < \varepsilon$ , for all  $\varepsilon > 0 $. Using the triangle inequality, 
\begin{equation*}
 | f_n - f_m| = | f_n - f + f - f_m | \leq | f_n - f| + | f - f_m|
\end{equation*}
for all $n, m > N$ for some $N > 0$. Replacing $\varepsilon$ by $\varepsilon/$2, we have $| f_n - f_m| < \varepsilon$. Therefore, $f_n$ is a Cauchy sequence.
\item Consider $L^2$ ([0,1]) $\cap$ C([0,1]), where $L^2$ is a set of continuous square integrable functions on $[0,1]$ and $C([0,1])$ is a set of continuous functions on $[0,1]$. Define a sequence 
\begin{equation*}
    f_n(x) = \left\{\begin{array}{lr}1, &  x = 0
        \\1-nx, &  0\leq x\leq \frac{1}{n}
        \\0 &  x \geq \frac{1}{n}
        \end{array}\right.
\end{equation*}
Using the inner product defined in Example~\ref{ex:square-integ}, and assuming $n > m$,
\begin{align*}
\|{f_n - f_m}\| &= \int_{x=0}^{1} (f_n(x)-f_m(x))^2 dx\\
&=\int_{0}^{\frac{1}{n}}(n-m)^2 x^2 dx+ \int_{\frac{1}{n}}^{\frac{1}{m}}m^2 x^2 dx + 0\\
&=(n-m)^2 \left(\frac{1}{3n^2}\right) + \frac{m^2}{3} \left(\frac{1}{m^3}-\frac{1}{n^3}\right)\\
&\leq \frac{1}{3m}
\end{align*}
Therefore, the sequence is Cauchy. As $n \to \infty, f_n$ becomes discontinuous and hence $\{f_n\}$ is not convergent. 
\end{enumerate}
\end{proof}
Now we are ready to define a Hilbert space.
\begin{definition}[Hilbert Space] An inner product space $H$ is called a Hilbert space if it is complete, i.e., if every Cauchy sequence is convergent, i.e., 
$\{f_n\}^{\mathrm{\infty}}_{n=1}$ is Cauchy in $H$ implies that there is an $f \in H$ such that $f_n \to f$.
\end{definition}
\noindent{\bf Remark.} 
A normed linear space is called a \textit{Banach} space if every Cauchy sequence converges. Hence every Hilbert space is also a Banach space. 

\subsection{Reproducing Kernel Hilbert Spaces}
\label{sec:hilbert:rkhs}
We now show that for every p.d. kernel $\varphi: X \times X \to \mathbb{C}$ we can construct a Hilbert space, $H$, associated with it in the sense that there is a mapping $f : X \rightarrow H$ such that, for every $x, y \in X$, $\varphi(x,y) = \langle f(x), f(y)\rangle$, where $\langle \cdot, \cdot \rangle$ is the inner product defined on $H$. This $H$ will be called the {\em Reproducing Kernel Hilbert Space} associated with $\varphi$, for reasons that will become clear shortly. 

First, given $x \in X$ and p.d. kernel $\varphi: X \times X \to \CC$, define univariate function $\varphi_x: X \to \CC$ as follows: for $y \in X$, $\varphi_x(y) = \varphi(x,y)$. 

If $\CC^X$ is the set of functions from $X$ to $\CC$, let $H_0$ be the linear subspace of $\CC^X$ generated by $\{ \varphi_x : x \in X\}$, i.e., if $f, g\in H_0$ then for some $n \in \NN$, $\{x_1, \ldots, x_n\} \subseteq X$, $\{c_1,\ldots, c_n\} \subseteq \CC$, $f = \sum_{j=1}^n c_j \varphi _{x_j}$ and, similarly, $g = \sum_{k=1}^m d_k \varphi _{y_k}$ for some $m \in \NN, \{y_1, \ldots, y_m\}\subseteq X$,  and $\{d_1, d_2, \ldots, d_n \} \subseteq \mathbb{C}$. We define a bivariate function $\langle \cdot, \cdot \rangle : H_0 \times H_0 \to \CC$ as follows:
\[\langle f,g \rangle = \sum_{j,k} c_j \overline{d_k} \varphi (x_j,y_k). \]
\paragraph{The reproducing property.}
To appreciate the definition of $\langle \cdot, \cdot\rangle$, let us define the $n \times m$ matrix $M = (\varphi(x_i, y_j) : i \leq n, j\leq m)$ and consider the row vectors $\bm{c} = (c_1, \ldots, c_n)$ and $\bm{d} = (d_1, \ldots, d_m)$. Then 
\[ \langle f, g \rangle = \bm{c} M \bm{d}^*.\]
Note that $\bm{c} M$ is a $1 \times m$ row vector whose $i$th coordinate is 
\[\sum_{j=1}^n c_j \varphi(x_j, y_i) = f(y_i).\]
Therefore 
\[\langle f, g \rangle = \sum_{i=1}^m \overline{d_i} f(y_i).\]
If $g = \varphi_y$, then we have that $\langle f, \varphi_y \rangle = f(y)$. And, in particular, if $f = \varphi_x$ and $g = \varphi_y$ we get the so called ``reproducing property'',
\[\langle \varphi_x, \varphi_y \rangle = \varphi_x(y) = \varphi(x,y).\]
Viewed from the other direction, using the Hermitian property of $\varphi$, note that $M \bm{d}^*$ is a $n \times 1$ column vector whose $j$th coordinate is 
\[\sum_{i=1}^m \overline{d_i} \varphi(x_j,y_i) = \overline{\sum_{i=1}^m d_i \varphi(y_i,x_j)} = \overline{g(x_j)}, \]
which tells us that 
\[\langle f, g \rangle = \sum_{i=1}^n c_i \overline{g(x_i)}.\]

\begin{proposition}
$\langle \cdot, \cdot \rangle$ is an inner product on $H_0$.
\end{proposition}
\begin{proof}
$\langle f,f \rangle \geq 0$ since $\langle f,f \rangle = \sum_{j,k} c_j \overline{c_k} \varphi (x_j,x_k) \geq 0$ as $\varphi$ is a p.d. kernel.
Linearity and the fact that the function is Hermitian are easy to see and so $\langle \cdot , \cdot\rangle$  is a semi-inner product.

By the reproducing property we have that 
\begin{align*}
|f(x)|^2 &= |\langle f,\varphi _x \rangle |^2\\ 
&\leq \ \| f \| ^2 \ \| \varphi _x\| ^2\\ 
& = \langle f,f \rangle . \varphi (x,x)
\end{align*}
since $\varphi$ is a p.d. kernel, $\varphi(x,x) \geq 0$. And $|f(x)|^2 > 0$ unless $f$ is identically 0. Therefor $\langle f,f \rangle > 0$ unless $f$ is identically 0 (i.e. the \textit{zero} for the vector space $H_0$) .
\end{proof}

So, we see that $H_0$ has an inner product defined on it, and this inner product has the ``reproducing property''. However it is not clear that $H_0$ is a Hilbert space. Nonetheless $H_0$ can be completed to give a Hilbert space $H$ in a standard way, and this $H$ is called the reproducing kernel Hilbert space associated with $\varphi$. 

\subsection{Mercer's theorem}

Finally, we discuss Mercer's theorem. We have already seen in Proposition~\ref{prp:sip-pd} that every semi-inner product is a p.d. kernel. Mercer's theorem gives us the converse: that every p.d. kernel is an inner product in some Hilbert space. Arguably the construction provided in Section~\ref{sec:hilbert:rkhs} has already established this fact. Here we present an alternate way of proving this. 

Generalizing Definition~\ref{def:ker-pd} we say that a bivariate function $\varphi: X \times X \to \mathbb{C}$ is said to satisfy Mercer's condition if for all square integrable functions $f$, we have  
\begin{equation*}
 \int_{X \times X} \varphi(x,y) f(x)\overline{f(y)} dx dy \geq 0.
\end{equation*}
Clearly a bivariate function that satisfies this definition is also a p.d. kernel as per Definition~\ref{def:ker-pd}. Mercer's theorem says that any bivariate function satisfying Mercer's condition can be expressed as an inner product of a separable Hilbert space. Here we prove a much simpler version of the theorem for p.d. kernels defined on finite sets, noting that the basic arguments are similar to that of the full theorem. 

\begin{theorem}[Mercer's Theorem for finite sets] Given a finite set $X$, a Hermitian function $\varphi: X \times X \to \mathbb{C}$ is a p.d. kernel iff there exists a Hilbert space ($ H ,\langle \cdot,\cdot \rangle_H$) and a mapping $f: X \to H$ such that $\varphi(x,y)= \langle f(x),f(y)\rangle_H$ for all $x,y \in X$.
\end{theorem}
\begin{proof}
Proposition~\ref{prp:sip-pd} establishes that an inner product is p.d. So we turn to the other direction.

Assume $\varphi$ is a p.d. kernel. Consider $A = \{ \varphi (x_i,x_j): i,j \leq \mid X \mid\}$, i.e., the Gram matrix of $X$ w.r.t. $\varphi$. Let $f: \mathbb{C}^{\mid X \mid} \to \mathbb{C}^{\mid X \mid}$ and define $f(x) = e_x$, where $e_x$ is a vector in $\mathbb{C}^{\mid X \mid}$ s.t. $e_x = 1$ at $x$th position and 0 otherwise. Consider the inner product $\langle \cdot, \cdot \rangle _A$ as defined in Example~\ref{ex:matrix-ip}. Then
\begin{equation*}
 \langle f(x), f(y) \rangle _A = \langle e_x, e_y \rangle_A = e_x A e_y = A_xy = \varphi(x,y)
\end{equation*}
\end{proof}
 
\section{Random Fourier Features}
\label{sec:rff}

In this section we discuss the technique introduced by Rahimi and Recht~\cite{rahimi-nips:2007} for efficient kernel computation. Our treatment is based on the lecture notes by Stephen Tu~\cite{tu:2016}.

\subsection{Fourier Transforms and Translation invariant kernels}

We begin by defining the Fourier transform of a function. Before that we recall that $L^1(X) = \{f: X \to \CC : \int_X |f(x)|dx < \infty\}$ is the set of summable complex-valued functions defined on a set $X$. Within this class we identify those functions which take non-negative real values and have the property that $\int_X f(x)dx = 1$. We call these probability distributions.
\begin{definition}[Fourier transform/Characteristic function]
Given a $d >0$ and $f\in L^1(\RR^d)$ the {\em Fourier transform of $f$}, $\hat{f}: \RR^d \to \CC$, is defined as
\begin{equation*}
   \hat{f}(\xi) = \int_{\RR^d} f(\bm{x}) e^{-i\xi^T \bm{x}}d\bm{x}.
\end{equation*}
If $f$ is a probability distribution and $Y$ is a random variable with distribution $f$ then $\hat{f}(\xi) = \ex{e^{i \xi Y}}$ is called the {\em characteristic function} of $Y$. 
\end{definition}
\noindent{\bf Remark.} In the case when $f$ is a probability distribution it is easy to check that $\hat{f}(0) = 1$. Conversely if $f$ takes non-negative values and $\hat{f}(0) = 1$ then $f$ is a probability distribution.
 
The Fourier transform has an important positive definiteness property.
\begin{proposition}
\label{prp:ft-pd}
Suppose, for some $d > 0$, that $\hat{f}$ is the Fourier transform of some $f \in L^1(\RR^d)$. Then the bivariate function $\varphi: \RR^d \times \RR^d$ defined as 
\[\varphi(\xi_1,\xi_2) = \hat{f}(\xi_1 - \xi_2),\]
is a p.d. kernel.
\end{proposition}
\begin{proof}
Consider some $n \in \mathbb{N}$, $\{c_1, \ldots, c_n\} \in \mathbb{C}$ and $\{\xi_,\ldots, xi_n\} \in \RR^d$. Then
 \begin{align*}
        \sum_{j,k=1}^{n}c_j \overline{c_k}\varphi(\xi_j, \xi_k) &= \sum_{j,k=1}^{n}c_j \overline{c_k} \hat{f}(\xi_j-\xi_k) \\
& = \sum_{j,k=1}^{n}c_j \overline{c_k} \int_{\RR^d} f(\bm{x}) e^{-i(\xi_j-\xi_k)^T \bm{x}}d\bm{x}\\
& = \int_{\RR^d} f(\bm{x}) \left\{\sum_{j,k=1}^{n}c_j \overline{c_k} e^{-i(\xi_j-\xi_k)^T \bm{x}}\right\}d\bm{x}\\
& = \int_{\RR^d} f(\bm{x}) \left\{\sum_{j,k=1}^{n}c_je^{-i\xi_j^T \bm{x}} \overline{c_k e^{-i\xi_k^T \bm{x}}}\right\}d\bm{x}\\
& = \int_{\RR^d} f(\bm{x}) \left|\sum_{j=1}^{n}c_je^{-i\xi_j^T \bm{x}}\right|^2d\bm{x} > 0
\end{align*}
\end{proof}

The bivariate function derived from a univariate function in Proposition~\ref{prp:ft-pd} suggests the following definition.
\begin{definition}[Translation invariant kernel]
A kernel $\varphi: X \times X \to \mathbb{C}$ is called {\em  translation invariant} if $\varphi(x_1,x_2) = g(x_1 - x_2)$ for some $g: X \to \CC$.
\end{definition}
Proposition~\ref{prp:ft-pd} provides a natural way of proving that a translation invariant kernel is p.d.: we can simply show that it is the Fourier transform of some function. Let us discuss this in the context of some examples. 
\begin{example}
\label{ex:pd-ti}
The following are examples of translation-invariant p.d. kernels.
\begin{enumerate}
    \item Gaussian kernel: $e^-\frac{\|x_1 - x_2\|^2}{2\sigma^2}$ for $\sigma > 0$
    \item Laplacian kernel: $e^-\frac{\|x_1 - x_2\|_1}{\sigma}$ for $\sigma > 0$
    \item Sinc kernel: $\sin{\frac{a(x_1 - x_2)}{\pi(x_1 - x_2)}}$ for $a > 0$ 
\end{enumerate}
\end{example}
\noindent{\bf Remarks.}
\begin{enumerate}
\item The fact that the Gaussian kernel is p.d. follows from Schonberg's Theorem (Theorem~\ref{thm:schonberg}) and Example~\ref{ex:gaussian-cnd} extended to $d$-dimensions. Later in this lecture we will present an alternate proof for the Gaussian kernel that goes via Proposition~\ref{prp:ft-pd}. 
\item From Proposition~\ref{prp:seminorm-cnd} we know that the 1-norm is c.n.d. so again applying Schonberg's Theorem  we see that the Laplacian kernel is p.d.  
\item For the sinc kernel note that if we define the function $$\mbox{rect}(x) = 
\left\{
\begin{array}{ll}
1, & -\frac{1}{2} \leq x \leq \frac{1}{2}\\
0, & \mbox{otherwise}
\end{array}
\right.
$$
and, for $a,b> 0$, $f_{a,b}(x) = b\cdot \mbox{rect}(x/a)$ then
\begin{align*}
\widehat{f_{a,b}}(\xi) &= b\int_{-a/2}^{a/2} e^{-i\xi x}dx\\ 
& = \frac{b}{i\xi} \left(e^{-\frac{ia\xi}{2}} - e^{\frac{ia\xi}{2}}\right) \\
& = \frac{2b}{\xi} \sin\left(\frac{a\xi}{2}\right).
\end{align*}
Clearly then $\sin(a\xi)/\pi \xi$ is $\widehat{f_{a,1/2\pi}}$ and hence, by Proposition~\ref{prp:ft-pd}, the Sinc kernel is p.d.
\end{enumerate}

The converse of Proposition~\ref{prp:ft-pd} is a celebrated theorem due to Bochner which states, roughly, that every translation-invariant p.d. kernel can be written as a Fourier transform of some function. Since stating this theorem in its full generality is not possible given the theory developed so far, we omit the statement. 

From Bochner's theorem we know that the Gaussian kernel should be expressible as a Fourier transform of some function. In this special case it is possible to explicitly determine the function without resorting to Fourier inversions. In fact we can show the remarkable fact that the Gaussian kernel can be expressed as the characteristic function of a random variable that is itself Gaussian.
\begin{proposition}
\label{prp:gaussian}
For $\sigma \in \RR$, if $Y \in \RR^d$ is a normally distributed multivariate random variable with mean 0 and covariance matrix  $1/\sigma^2 \cdot I$ then 
\[ e^{-\frac{\|\xi\|_2^2}{2\sigma^2}} = \ex{e^{-i \xi Y}}.\]
\end{proposition}
\begin{proof}
First we characterize the Fourier transform of a univariate Gaussian.
\begin{lemma}
\label{lem:univ-gauss}
Let $f(x) = e^{-\frac{zx^2}{2}}$ for some positive $z>0$. Then 
\begin{equation*}
    \hat f(\xi) = (2\pi)^\frac{1}{2}z^{\frac{-1}{2}}  e^{\frac{-\xi^2}{2z}}
\end{equation*}
\end{lemma}
\begin{proof}[Proof of Lemma~\ref{lem:univ-gauss}]
Differentiating under the Fourier transform integral we have
\[
    \frac{d}{d\xi} \hat f(\xi)  = \frac{\mathrm{d}}{\mathrm{d}\xi} \int_{-\infty}^{\infty} e^{-\frac{zx^2}{2}} e^{-i\xi x} dx
    =\int_{-\infty}^{\infty} e^{-\frac{zx^2}{2}}(-i\xi) e^{-i\xi x} dx =\frac{i}{z} \int_{-\infty}^{\infty} e^{-i\xi x} \frac{\mathrm{d}}{\mathrm{d}x} e^{-\frac{zx^2}{2}}.
\]
Integrating by parts we get 
\[
  \frac{d}{d\xi} \hat f(\xi) =\frac{i}{z}e^{-i\xi x} \left[e^{-i\xi x} e^{-\frac{zx^2}{2}}|_{-\infty}^\infty -  \int_{-\infty}^{\infty} (-i\xi) e^{-i\xi x} e^{-\frac{zx^2}{2}}\right]
    =\frac{\xi}{z} \hat f(\xi).
\]
This is an ordinary differential equation and $\hat f(\xi) = Ce^{\frac{-\xi^2}{2z}}$ is a solution for any constant $C$. By satisfying the boundary condition, we have $C = \hat f(0) = \int e^{-\frac{zx^2}{2}} dx$. Since $\hat f(0)$ is the normalization constant of a $N(0,\frac{1}{z})$ distribution, therefore, $C = (2\pi)^\frac{1}{2}z^{\frac{-1}{2}}.$
\end{proof}

Now consider the $d$-dimensional Gaussian function $f(x)= e^-\frac{z{\|x\|_2}^2}{2}$ for $z>0$. Using Lemma~\ref{lem:univ-gauss} and taking a Fourier transform we have
\[
    \hat{f}(x) = \int_{\mathbb{R}^d} e^{-\frac{z{\|x\|_2}^2}{2}} e^{-i\xi x} dx =  \int_{\mathbb{R}^d}  \prod\limits^{d}_{i=1} \int_{-\infty}^{\infty} e^{-\frac{zx_i^2}{2}} e^{-i\xi_i x_i} dx_i
    =\prod\limits^{d}_{i=1} \hat{f}(\xi_i) = (2\pi)^{\frac{d}{2}} z^{-\frac{d}{2}} e^{\frac{{-{\|\xi\|_2}^2}}{2z}}
\]
Setting $z=\sigma^2$ and $\gamma = 1/\sigma$, we have
\begin{equation*}
    e^{\frac{{-{\|\xi\|_2}^2}}{2\sigma^2}} = \frac{1}{(2\pi\gamma^2)^{-\frac{d}{2}}} \int_{\mathbb{R}^d} e^{-i\xi x} e^{-\frac{{\|x\|_2}^2}{2\gamma^2}} dx 
\end{equation*}
We now see that RHS is simply the expectation over a random variable distributed as a multivariate Gaussian for $N(0,\gamma^2I)$.
\end{proof}

\subsection{The Random Fourier Feature method of Rahimi and Recht~\cite{rahimi-nips:2007}}
\label{sec:rff:rr}
In~\cite{rahimi-nips:2007}, Rahimi and Recht suggested the following mapping:
\begin{definition}[Random Fourier Features] 
\label{def:rff}
Given a set of training points $T = \{\bm{x}_1,\ldots,\bm{x}_n\} \subseteq \mathbb{R}^d$ such that $d>1$. For some $\sigma \in \RR$ and $D>0$, we construct a random feature map $\phi: \mathbb{R}^d \to [0,1]^D$ as 
\begin{equation*}
\phi(\bm{x}) = \left [ \begin{array} {c}
\cos{(\bm{w}_t^T+b_1)}\\
\cdot\\
\cdot\\
\cdot\\
\cos{(\bm{w}_D^T+b_D)}
\end{array} \right ]
,\bm{w}_i \overset{\mathrm{iid}}{\sim} \mathrm{N}(0,\sigma^2I), b_i \overset{\mathrm{iid}}{\sim} \mathrm{Unif}([0,2\pi])
\end{equation*}
\end{definition}
Note that all $D$ dimensions are independent of each other.  
\begin{fact}
\label{fct:rr}
Suppose $\bm{x}, \bm{y} \in \RR^d$. If $\bm{w} \in \RR^d$ is chosen randomly with distribution $\mathrm{N}(0,1/\sigma^2\cdot I)$ and $b$ is chosen uniformly at random from $[0,2\pi]$ then  \[\exs{\bm{w},b}{\phi(\bm{x})_{\bm{w},b} \phi(\bm{y})_{\bm{w},b}} = e^{-\frac{\|\bm{x} - \bm{y}\|_2^2}{2\sigma^2}}.\]
\end{fact}
\begin{proof}
From Proposition~\ref{prp:gaussian} we know that $\varphi(\bm{x},\bm{y}) = e^{-\frac{\|\bm{x} - \bm{y}\|_2^2}{2\sigma^2}} = \exs{\bm{w}}{e^{-i(\bm{x}-\bm{y})\bm{w}}}$. Note that
\begin{align*}
e^{-i(\bm{y}-\bm{x})\bm{w}}     =& (\cos{(\bm{w}^T\bm{x})} + i\sin{(\bm{w}^T\bm{x})})(\cos{(\bm{w}^T\bm{y})} - i\sin{(\bm{w}^T\bm{y})})\\
    =& \cos{(\bm{w}^T\bm{x})}\cos{(\bm{w}^T\bm{y})}+\sin{(\bm{w}^T\bm{x})}\sin{(\bm{w}^T\bm{y})}\\
    &+i(\sin{(\bm{w}^T\bm{x})}\cos{(\bm{w}^T\bm{y})} - \cos{(\bm{w}^T\bm{x})}\sin{(\bm{w}^T\bm{y})})\\ 
    =& \cos{(\bm{w}^T(\bm{x}-\bm{y})}) + i(\sin{(\bm{w}^T(\bm{x}-\bm{y})}).
\end{align*}
Taking expectations on both sides we and recalling that the Gaussian kernel is real-valued we get that $\exs{\bm{w}}{e^{-i(\bm{y}-\bm{x})\bm{w}}} = \exs{\bm{w}}{\cos (\bm{w}^T(\bm{x}-\bm{y}))}$. Further note that 
\begin{align*}
    \exs{\bm{w},b}{\phi(\bm{x})_{\bm{w},b} \phi(\bm{y})_{\bm{w},b}} &= \exs{\bm{w},b}{2 \cos{(\bm{w}^T\bm{x}+b)}\cos{(\bm{w}^T\bm{y}+b)}}\\
& = \exs{\bm{w},b}{\cos{(\bm{w}^T(\bm{x}-\bm{y}))}} + \exs{\bm{w},b}{\cos{(\bm{w}^T(\bm{x}-\bm{y})+2b)}}\\
    &=\exs{\bm{w},b}{\cos{(\bm{w}^T(\bm{x}-\bm{y}))}} + \exs{\bm{w}}{\exs{b}{\cos{(\bm{w}^T(\bm{x}-\bm{y})+2b)}}}
\end{align*}
Since  $\int_{0}^{2\pi}\cos{(a+2x)}dx = 0$ for any  $a\in\mathbb{R}$, the second term on the RHS is 0 and so we have that 
\[\exs{\bm{w},b}{\phi(\bm{x})_{\bm{w},b} \phi(\bm{y})_{\bm{w},b}} = \exs{\bm{w}}{\cos (\bm{w}^T(\bm{x}-\bm{y}))}  = \exs{\bm{w}}{e^{-i(\bm{x}-\bm{y})\bm{w}}} = \varphi(\bm{x},\bm{y}).\]
\end{proof}

From Fact~\ref{fct:rr} it is clear that each coordinate of the random mapping $\phi$ gives us a Gaussian kernel computation in expectation. However a single random choice may be far from the mean. Rahimi and Recht choose $D$ coordinates to reduce the error, i.e., they try to empirically estimate the mean. They define an inner product on $[0,1]^D$ as follows: $\langle x, y\rangle = \frac{1}{D} \sum_{i=1}^D x(i)y(i)$. Fact~\ref{fct:rr} tells us that  $\langle x, y\rangle$ is in fact an emprirical estimate of the value $\varphi(x,y)$. This allows us to use simple dot products in $\RR^D$ to estimate the kernel computation in $\RR^d$ required for a learning task, as discussed in Section~\ref{sec:motivation}.

In Section~\ref{sec:concentration-scalar} we will discuss concentration results that help us bound the quality of the empirical estimate in terms of the parameter $D$. In Section~\ref{sec:conc:matrix} we show that the Gram matrix defined by $\phi$ on the training set via the inner product $\langle \cdot, \cdot \rangle$ is a good approximation of the Gram matrix defined by the Gaussian kernel.

\section{Concentration for scalar random variables}
\label{sec:concentration-scalar}
We now discuss the phenomenon of ``concentration'', i.e., the fact that sums of independent random variables are often close to the expectation of their sums. The theory is vast, so we only discuss that part of it that is relevant to proving some properties of Rahimi and Recht's method~\cite{rahimi-nips:2007} presented in Section~\ref{sec:rff:rr}. Specifically we are interested in proving Hoeffding's inequality which applies to random variables that take values in a bounded range. The material in this section is based on the presentation in the book by Boucheron, Lugosi and Massart~\cite{boucheron:2016}.

\subsection{The Cramer-Chernoff method}
\label{sec:concentration-scalar:chernoff}
Hoeffding's inequality and several other concentration results depend on the following simple result that we call the ``extended'' Markov's inequality.
\begin{lemma}[Extended Markov's Inequality]
\label{lem:markov}
If $X$ is a real-valued random variable and $\phi: \RR \to \RR_+$ is a non-decreasing non-negative valued function, then 
\[ \pr{X \geq t}
        \leq \frac{\ex{\phi(X)}}{\phi(t)}.\]
\end{lemma}
\begin{proof}
Observing first that since $\phi$ is non-decreasing, we have, for any random variable $X$ and $t\in \RR$, that $\pr{X > t} \leq \pr{\phi(X) > \phi(t)}$.

Now, suppose $Y$ is a non-negative valued random variable with probability density function $f$, then
\[ \ex{Y} = \int_0^\infty x f(x)dx \geq \int_t^\infty x f(x)dx \geq t \int_t^\infty f(x)dx 
            = t\cdot \pr{Y \geq t}.\]
The result follows. 
\end{proof}

\noindent \textbf{Remark.} Chebyshev's inequality follows from Lemma~\ref{lem:markov}. Taking $Y = Z - \ex{Z}$ and $\phi(x) = x^2$ gives us 
\begin{equation*}
\pr{\left|Z - \ex{Z}\right| \geq t} \leq \frac{\mbox{Var}(Z)}{t^2},
\end{equation*}
since, by definition, $\mbox{Var}(Z) = \ex{\left|Z - \ex{Z}\right|^2}$.

\begin{definition}[Moments of a Random Variable] 
Given a random variable $X$, $\ex{X^k}$, $k \geq 0$ is called the $k^{th}$ {\em moment} of $X$.
\end{definition}

\noindent \textbf{Remarks.}
\begin{enumerate}
    \item $\mbox{Var}(X)$ i.e the variance of $X$ contains 2 moments in it since $\mbox{Var}(X) = \ex{X^2} - \ex{X}^2$. 
    \item It is not difficult to see that if $\ex{|X|^k} < \infty$ for $k > 1$ then $\ex{|X|^{k-1}} <\infty$.
    \item However the converse of the previous remark is not true. For example consider a discrete random variable $X$, $X > 0$, such that $\pr{X=i} = \frac{c}{i^3}$ where  $c = 1/(\sum_{i=1}^\infty 1/i^3)$. In this case  $\ex{X} = \sum_{i=1}^\infty \frac{c}{i^2} < \infty$ but $\ex{X^2} = \sum_{i=1}^\infty \frac{c}{i}$ is unbounded.
    \item The more higher moments we have, the tighter the bound we can get using Lemma~\ref{lem:markov}. In general, for $k \geq 1$,
    \begin{equation*}
        \pr{|Z - \ex{Z}| \geq t} \leq \frac{\ex{|Z - \ex{Z}|^k}}{t^k}.
    \end{equation*}
\end{enumerate}

When all moments are finite it is convenient to wrap up all the moment information in a single function called the moment generating function.
\begin{definition}[Moment generating function and Cumulant generating function]
Given a real-valued random variable $Z$, we call 
\[M_Z(\lambda) = \ex{e^{\lambda Z}}\]
the {\em moment generating function} of $Z$. Further we call
\[\psi_Z(\lambda) = \log M_Z(\lambda) = \ex{e^{\lambda Z}}\]
the {\em cumulant generating function} of $Z$.
\end{definition}
\noindent{\bf Remark.} The $k$-th moment of $Z$ can be retrieved by evaluating the $k$-th derivative of $M_Z(\lambda)$ at $\lambda = 0$.

The cumulant generating function has an important additivity property.
\begin{proposition}
\label{prp:cumulant-additivity}
If $Z_1, \ldots, Z_n$ are independent real-valued random variables and $Z = \sum_{i=1}^n Z_i$ then 
\[\psi_Z(\lambda) = \sum_{i=1}^n \psi_{Z_i}(\lambda).\]
\end{proposition}
\begin{proof}
The result follows by observing that 
\[M_Z(\lambda) = \ex{e^{\lambda\sum_{i=1}^n Z_i}} = \ex{\prod_{i=1}^n e^{\lambda Z_i}} = \prod_{i=1}^n \ex{e^{\lambda Z_i}},\]
where the last equality holds only in the case that $Z_1, \ldots, Z_n$ are independent.
\end{proof}

With these definitions we are ready to state Chernoff's inequality.
\begin{lemma}[Chernoff's Inequality]
\label{lem:chernoff}
For a real valued random variable $Z$ and any $t > 0$, 
\[\pr{Z > t} \leq e^{-\psi^*(t)},\]
where 
\[\psi^*(t) = \sup_{\lambda > 0} \left(\lambda t - \psi_Z(\lambda)\right) =\sup_{\lambda > 0} \left(\lambda t - \log \ex{e^{\lambda Z}}\right).\]
If $t > \ex{Z}$ then 
\[\psi^*(t) = \sup_{\lambda \in \RR} \left(\lambda t - \log \ex{e^{\lambda Z}}\right).\]
\end{lemma}
\begin{proof}
Noting that the function $\phi(x) = e^{\lambda x}$ takes non-negative values and is non-decreasing whenever $\lambda > 0$, Markov's inequality (Lemma~\ref{lem:markov}) gives us that 
\[\pr{Z > t} \leq e^{-\lambda t} \ex{e^{\lambda Z}}\]
for every $\lambda > 0$. The RHS is minimized by choosing the value of $\lambda$ that maximizes $\lambda t - \log \ex{e^{\lambda Z}}$.

In case $t > \ex{Z}$ then, by Jensen's inequality, we have that
\begin{gather*}
    \log\ex{e^{\lambda Z}} \geq \ex {\log e^{\lambda Z}} = \lambda \ex{Z}.
\end{gather*}
This tells us that if $\lambda < 0$ then 
\begin{gather*}
\lambda t - \log\ex{e^{\lambda Z}} \leq \lambda t -  \lambda \ex{Z} < 0,
\end{gather*}
and so the maximization can be formally extended over the entire real line.
\end{proof}
\noindent{\bf Remark.} From Proposition~\ref{prp:cumulant-additivity} we can conclude that if $Z = \sum_{i=1}^n Z_i$ for an independent collection $Z_1, \ldots, Z_n$ then 
\[\psi^*(t) = \sup_{\lambda > 0} \left(\lambda t - \sum_{i=1}^n \log \ex{e^{\lambda Z_i}}\right).\]

\subsection{Hoeffding's inequality}
In this section we will prove Hoeffding's inequality and show how it is used in~\cite{rahimi-nips:2007}. Hoeffding's inequality is an application of the Cramer-Chernoff method to sums of independent random variables that take values within a bounded range. The logarithm of the deviation of the sum from its expectation varies inversely with the size of the range. 

\begin{theorem}[Hoeffding's inequality] 
\label{thm:hoeffding}
Let $X_1, X_2,\ldots, X_n$ be independent random variables such that $X_i$ takes its values in $[a_i, b_i]$, $1 \leq i \leq n$. Let 
$S = \sum_{i=1}^n\left(X_i - \ex{X_i}\right)$.    Then, for every $t > 0$,
    \begin{equation*}
\pr{S \geq t} \leq \exp \left(\frac{-2t^2}{\sum_{i=1}^n \left(b_i - a_i\right) ^ 2}\right).
    \end{equation*}
\end{theorem}
\begin{proof}
The proof relies on the following characterisation of the cumulant generating function of a bounded random variable. 
\begin{lemma}[Hoeffding's lemma]
\label{lem:hoeffding}
If $Z$ is a random variable with $\ex{Z} = 0$ that takes values in $[a,b]$ then 
\[\psi_Z(\lambda) \leq \frac{\lambda^2(b-a)^2}{8}\]
for all $\lambda \in \RR$.
\end{lemma}
\begin{proof}[Proof of Lemma~\ref{lem:hoeffding}]
Since $\psi_Z(\lambda)$ is at least twice differentiable at $\lambda =0$, by Taylor's theorem we know that there is a $\theta \in [0,\lambda]$ such that
\begin{equation}
\label{eq:theta}
\psi_Z(\lambda) = \psi_Z(0) + \psi_Y'(0) + \psi_Y''(\theta).
\end{equation}
$\psi_Z(0) = \log 1 = 0$. And, for any $\lambda \in \RR$ we have that
\[ \psi_Z'(\lambda) = \frac{\ex{Z e^{\lambda Z}}}{\ex{e^{\lambda Z}}} = e^{-\psi_Z(\lambda)}\ex{Ze^{\lambda Z}}.\]
At $\lambda = 0$ this reduces to $\ex{Z}$ which is 0 by assumption. So, we turn to the second derivative. Differentiating $\psi_Z'(\lambda)$ again we get
\[\psi_Z''(\lambda) = \frac{\ex{Z^2 e^{\lambda Z}}}{\ex{e^{\lambda Z}}} - \left(\frac{\ex{Z e^{\lambda Z}}}{\ex{e^{\lambda Z}}}\right)^2 = e^{-\psi_Z(\lambda)}\ex{Z^2 e^{\lambda Z}} - \left(e^{-\psi_Z(\lambda)}\ex{Ze^{\lambda Z}}\right)^2.\]
We will show that the RHS is the variance of a random variable that takes values in $[a,b]$. For any given $\lambda \in R$ define a random variable $X_\lambda$ which takes value $x$ with probability $e^{-\psi_Z(\lambda)}\cdot e^{\lambda x} \cdot \pr{Z = x}$. By the definition of $\psi_Z(\lambda)$ this is a probability ditribution. Since $\pr{Z= x}$ is 0 outside $[a,b]$, clearly $X_\lambda$ takes values only in $[a,b]$. The variance of $X_\lambda$ is 
\[\mbox{Var}(X_\lambda) = \int_a^b x^2 e^{-\psi_Z(\lambda)}\cdot e^{\lambda x} \cdot \pr{Z = x}dx - \int_a^b x e^{-\psi_Z(\lambda)}\cdot e^{\lambda x} \cdot \pr{Z = x}dx, \]
which is equal to $\psi_Z''(\lambda)$.

For any r.v. $Y$ that takes values in $[a,b]$, we know that $|Y - (b+a)/2| \leq (b-a)/2$ with probability 1 since the distance of any $x \in [a,b]$ from the mid-point of the interval is at most half the length of the interval. This immediately tells us that 
\[\mbox{Var}(Y) = \mbox{Var}\left(\left|Y - \frac{b+a}{2}\right|\right) \leq \frac{(b-a)^2}{4}.\]

Consider the $\theta$ defined in~\eqref{eq:theta}. Since $\psi_Z''(\theta) = \mbox{Var}(X_\theta)$ and $X_\theta$ is also a r.v. that takes values in $[a,b]$, it's variance is also at most $(b-a)^2/4$ and so the result follows.
\end{proof}

Note that if $X$ takes values in $[a,b]$ then so does $X - \ex{X}$. Hence for each $i$, $1 \leq in \leq n$, by Lemma~\ref{lem:hoeffding}, $\psi_{X - X_i}(\lambda) \leq \lambda^2(b_i - a_i)^2/8$. So, since $X_1,\ldots, X_n$ are independent Proposition~\ref{prp:cumulant-additivity} tells us that 
\[\psi_S(\lambda) \leq \sum_{i=1}^n \frac{\lambda^2(b_i - a_i)^2}{8}, \]
and so, by Chernoff's inequality (Lemma~\ref{lem:chernoff})
\[\pr{S > t} \leq \exp \left\{ -\sup_{\lambda \in \RR}\left(\lambda t - \sum_{i=1}^n \frac{\lambda^2(b_i - a_i)^2}{8}\right) \right\}.\]
By elementary calculus we find that the supremum on the right is achieved at 
\[\lambda^* = \frac{4t}{\sum_{i=1}^n (b_i - a_i)^2}.\]
Putting this value back into the previous expression gives us the result. 
\end{proof}

\subsection{A tail bound for Random Fourier Features}
We now apply Hoeffding's inequality (Theorem~\ref{thm:hoeffding}) to Random Fourier Features (Definition~\ref{def:rff}).
\begin{example}[A tail bound for Random Fourier Feautures]
\label{ex:rff-hoeffding}
Given a set of training points $T= \{\bm{x}_1,\ldots,\bm{x}_n\} \subseteq \mathbb{R}^d$ such that $d>1$. For some $\sigma \in \RR$ and $D>0$, we construct a random feature map $\phi: \mathbb{R}^d \to [0,1]^D$ as given in Definition~\ref{def:rff}. Define an inner product on $[0,1]^D$ as follows: $\langle \bm{x}, \bm{y}\rangle = \frac{1}{D} \sum_{k=1}^D \bm{x}(k)\bm{y}(k)$. Then, for all $\bm{x}_i, \bm{x}_j \in T$, and any $\delta, \epsilon > 0$
\[\pr{\left| \langle \phi(\bm{x}_i), \phi(\bm{x}_j) \rangle - e^{-\frac{\|\bm{x}_i - \bm{x}_j\|_2^2}{2\sigma^2}} \right| > \varepsilon} < \delta,\]
if
\[D \geq \frac{16}{\varepsilon^2} \log \frac{n}{\delta}.\]
\end{example}
\begin{proof}
From Fact~\ref{fct:rr} we know that $\ex{\phi(\bm{x}_i)(k)\phi(\bm{x}_j)(k)} =  e^{-\frac{\|\bm{x}_i - \bm{x}_j\|_2^2}{2}}$. Set $Z_k =  \phi(\bm{x}_i)(k)\phi(\bm{x}_j)(k)$, $1 \leq k \leq D$. These form an independent colleciton and each $Z_k$ takes values in $[-2,2]$.  So, by Hoeffding's inequality (Theorem~\ref{thm:hoeffding}) we have that 
\[\pr{\sum_{k=1}^D Z_k - \ex{Z_k} > D\varepsilon} \leq e^ {-D\varepsilon^2/8}. \] 
Setting $$e^{- D\varepsilon^2/8} < \frac{\delta}{n^2},$$
and solving for $D$ we get $$D > \frac{16}{\varepsilon^2} \log \frac{n}{\delta}.$$ Since the probability of the estimate being far from the mean is at most $\delta/n^2$ for a single pair therefore, by the union bound, the estimate for any of the $n^2$ pairs  $(\bm{x}_i,\bm{x}_j), 1 \leq i,j\leq n$ far from the mean with probability at most $\delta$.
\end{proof}

Suppose we say that $\bar{R}(D)$ is the $n \times n$ matrix with $\bar{R}(D)_{ij} = \langle \bm{x}_i, \bm{x}_j$ and if $G$ is the Gram matrix of $T$ w.r.t. the Gaussian radial basis function kernel, then Example~\ref{ex:rff-hoeffding} tells us that with probability at least $1 - \delta$ each entry of $R(D)$ is within $\varepsilon$ of each entry of $G$ if we choose the value of $D$ mentioned above. But what about $\| R \|$ versus $\| G\|$? How do these compare? And, is the dependence of $D$ on the size of $T$ needed? To address these questions we have to turn to concentration inequalities for matrix-valued random variables.

\section{Matrix concentration inequalities}
\label{sec:conc:matrix}
In this section we will introduce the matrix analog of the Cramer-Chernoff method. We will then use this method to prove the Matrix Bernstein Inequality and apply this inequality to Random Fourier Features. But first we introduce some background needed for this development. The treatment in this section closely follows the monograph by Tropp~\cite{tropp-ftml:2015}.

\subsection{Background: Matrix theory}
\label{sec:concmatrix:background-LA}

We collect some facts and definitions from the theory of matrices that will be relevant to our development. 
\subsubsection{Hermitian matrices and the p.d. partial order.}
We denote by $\MM_{d_1\times d_2}$ the set of all $d_1 \times d_2$ matrices with complex entries and by $\HH_d$ the set of all $d \times d$ Hermitian matrices. $\MM_{d_1\times d_2}$ is equipped with the {\em Frobenius norm}:
\[ \|A\|_F = \sum_{j=1}^{d_1}\sum_{k=1}^{d_2} |A_{jk}|^2, \mbox{ for } A \in \MM_{d_1\times d_2}.\]

We know that any $A \in \HH_d$ has can be decomposed as $A = Q \Lambda Q^*$ where $Q \in \MM_d$ is a unitary matrix (i.e. $QQ^* = I$) and $\Lambda$ is a diagonal matrix with real entries. The real entries of $\Lambda$ are called the {\em eigenvalues} of $A$. We denote by $\lmin(A)$ and $\lmax(A)$ the minimum and maximum of these eigenvalues. For a Hermitian matrix $A\in \HH_d$, the {\em spectral norm} can be defined in terms of these two extreme eigenvalues. 
\[ \|A\| = \min_{\bm{x} \in \CC^d} \frac{\|\bm{x}^* A \bm{x}\|}{\|\bm{x}\|^2} = \max\{\lmax(A), -\lmin(A)\}.\]
We state some simple facts about extreme eigenvalues.
\begin{fact} 
\label{fct:positive-homogenous}
The following hold for the maps $\lmin, \lmax: \HH_d \to \RR$ and any $A \in \HH_d$:
\begin{enumerate}
\item $\lmin, \lmax$ are {\em positive homogenous}, i.e., for $\alpha > 0$, $\lmin(\alpha A) = \alpha \lmin(A)$ and $\lmax(\alpha A) = \alpha \lmax(A)$.
\item $\lmin(-A) = -\lmax(A)$
\end{enumerate}
\end{fact}

\begin{definition}
The trace of a square matrix, denoted $\tr$, is the sum of its diagonal entries, i.e., for $A \in M_{d \times d}$, 
\[\tr A = \sum_{i=1}^d A_{ii}.\]
\end{definition}
We note some important facts about the trace.

\noindent{\bf Remarks.}
\begin{enumerate}
\item The trace is {\em unitarily invariant}, i.e., for any $A \in M_{d\times d}$ and unitary $Q$, $\tr A = \tr QAQ^*$. We omit the proof, referring the reader to any linear algebra text.
\item From the previous remark it follows that for $A \in \HH_d$, $\tr A$ is equal to the sum of the eigenvalues of $A$.
\item From the previous remarks it follows that if $A \in \HH_d$ is p.d. then $\lmax(A) \leq \tr A$. 
\end{enumerate}

We now define the {\em positive definite partial order}, $\preceq$, on $\HH_d$. Given $A, B \in \HH_d$, we say $A \preceq B$ if $B - A$ is p.d. Accordingly, we say that $A$ is p.d. if $A \succeq \bm{0}$ and $A$ is strictly p.d. if $A \succ \bm{0}$.
The p.d. partial order has the following important property which is easy to prove.
\begin{proposition}[Conjugation rule]
\label{prp:conjugation}
Given $A, B \in \HH_{d_1}$ and $C \in \MM_{d_1\times d_1}$, if $A \preceq B$ then $CAC^* \preceq CBC^*$.
\end{proposition}

\subsubsection{Intrinsic dimension of a Hermitian matrix}
As we know, the rank of matrix is an integer value and can be discontinuous, i.e., a small perturbation in the values of the matrix can lead to a jump in rank. We define a continuous notion of rank that will be useful for our analysis:
\begin{definition}[Stable Rank]
\label{def:srank}
For a matrix $B$, the {\em stable rank} is defined as:
$$ \mathrm{srank}(B) = \frac{\lVert B \rVert_{F}^{2}}{\lVert B \rVert^2}$$.
\end{definition}
\noindent{\bf Remarks.}
\begin{enumerate}
    \item It is known that $ \lVert B \rVert_{F}^2 = \sum_{i} \sigma_i^2$ where the $\sigma_i$ are the singular values of $B$. On the other hand the spectral norm is equal to the square of the largest singular value which is contained in the summation for $ \lVert B \rVert_{F}^2$. Hence one can see that $\mathrm{srank}(B) \geq 1$ for all $B$.
    \item Since the rank of a matrix is equal to the number of non-zero singular values we can say that $\sum_{i} \sigma_i^2 \leq k \left(\max_i \sigma_{i}\right)^2$, where $k$ is the rank of the matrix. Hence one can see that, $\mathrm{srank}(B) \leq \mathrm{rank}(B)$
\end{enumerate}
For p.d. matrices we also define a notion of dimensionality
\begin{definition}
\label{def:intdim}
For a p.d. Hermitian matrix $A$, we say that the {\em intrinsic dimension} of $A$ is 
$$\mathrm{intdim}(A) = \frac{\text{tr}(A)}{\lVert A \rVert}.$$
\end{definition} 
In fact the intrinsic dimension is nothing more than the stable rank of $\sqrt{A}$. 
$$ \mathrm{srank}(A) = \frac{\lVert \sqrt{A} \rVert_{F}^{2}}{\lVert \sqrt{A} \rVert^2} = \frac{\sum_{i} \sqrt{\lambda_i}^{2}}{\sqrt{\lmax(A)}^{2}} = \frac{\text{tr}(A)}{\max(A)} = \frac{\text{tr}(A)}{\lVert A \rVert}. $$

\subsubsection{Functions of Hermitian matrices}
We now describe how to apply functions defined on the real line to Hermitian matrices. 
\begin{definition}[Standard matrix function of a Hermitian matrix]
Let $A \in \HH_d$ be a matrix whose eigenvalues are contained in an interval $I$ of $\RR$ and let $f: I \to \RR$ be a function. Then we define the matrix $f(A) \in \HH_d$ as follows:
\begin{enumerate}
\item If $A$ is a real diagonal matrix then $f(A)$ is a real diagonal matrix with $f(A)_{ii} = f(A_{ii})$.
\item Otherwise, since $A$ can be decomposed as $Q \Lambda Q^*$ where $\Lambda$ is a real diagonal matrix, $f(A) = Q f(\Lambda) Q^*$.
\end{enumerate}
\end{definition}
The following proposition follows directly from the definition of standard matrix functions.
\begin{proposition}[Spectral Mapping Theorem]
\label{prp:spectral-mapping}
Let $A \in \HH_d$ be a matrix whose eigenvalues are contained in an interval $I$ of $\RR$ and let $f: I \to \RR$ be a function. If $\lambda$ is an eigenvalue of $A$ then $f(\lambda)$ is an eigenvalue of $f(A)$. 
\end{proposition}
We now discuss the conditions under which an ordering between real functions transfers to their matrix counterparts.
\begin{proposition}[Transfer rule]
Let $A \in \HH_d$ be a matrix whose eigenvalues are contained in an interval $I$ of $\RR$ and let $f,g: I \to \RR$ be real-valued functions. If $f(a) \leq g(a)$ for all $a \in I$ then $f(A) \preceq g(A)$.
\end{proposition}
\begin{proof}
Decompose $A$ as $Q\Lambda Q^*$. Since $ g(\Lambda) - f(\Lambda)$ is a real diagonal matrix with non-negative entries, so $f(\Lambda) \preceq g(\Lambda)$. Applying the Conjugation rule (Proposition~\ref{prp:conjugation}) we get $Qf(\Lambda)Q^* \preceq Qg(\Lambda)Q^*$ and the result follows.
\end{proof}

To prove bounds we are interested in the monotonicity property of various functions. Monotonicity itself comes in two flavours in the matrix setting: monotonicity of trace functions and monotonicity of functions in the p.d. partial order. The first one is easier to establish.
\begin{proposition}[Monotone trace functions]
\label{prp:monotone}
Let $f: I \to \RR$ be a non-decreasing function on an interval $I \subseteq \RR$ and let $A$ and $B$ be matrices whose eigenvalues are contained in $I$. Then, $A \preceq B$ implies that $\tr f(A) \leq \tr f(B)$. 
\end{proposition}
\begin{proof}
If $\lambda_i(A)$ is the $i$th eigenvalue of $A$, by the Courant-Fischer theorem it can be shown that whenever $A \preceq B$ then $\lambda_i(A) \leq \lambda_i(B)$ for all $i$. This, in turn implies that for any non-decreasing function $f$, $f(\lambda_i(A)) \leq f(\lambda_i(B))$ and consequently $\tr A \leq \tr B$. 
\end{proof}
In particular we will need the following corollary
\begin{corollary}
\label{cor:trexp}
If $A \preceq B$ then $\tr \exp(A) \leq \tr \exp (B)$.
\end{corollary}
A stronger class of functions that monotone trace functions are {\em operator monotone functions}, i.e., functions that preserve the p.d. partial order. A general characterization like that of Proposition~\ref{prp:monotone} is difficult to give for such function but we mention one important operator monotone function that we will be using.
\begin{proposition}
\label{prp:log-monotone}
If $A \preceq B$ then $\log A \preceq \log B$.
\end{proposition}
The proof proceeds by showing that the negative inverse is operator monotone and then integrating it to show that log is operator monotone. We refer the reader to Section 8.4.2 of~\cite{tropp-ftml:2015} for details.

\subsubsection{Probability with matrices}

An $n \times m$ random matrix $A$ can be viewed as simply a collection of $nm$ scalar random variables appropriately indexed, i.e., as $\{A_{ij}: 1 \leq i \leq n, 1\leq j \leq m\}$. Accordingly the expectation $\ex{A}$ of $A$ is a matrix such that $\ex{A}_{ij} = \ex{A_{ij}}$. An important fact about the expectation is that it preserves the p.d. partial order for .
\begin{fact}
Given p.d. Hermitian $X, Y$ such that $X \preceq Y$ with probability 1, $\ex{X} \preceq \ex{Y}$.
\end{fact}
This follows from the fact that $X - Y \succeq \bm{0}$ with probability 1 and that the expectation of $X - Y$ is a convex combination of p.d. Hermitian matrices which is also p.d. (see Proposition~\ref{prp:convexcone}).

Given a random Hermitian matrix $X$, we can define a matrix-valued variance of $X$ in a manner similar to that for scalar random variable, i.e.,
\[\var{X} = \ex{(X - \ex{X})^2} = \ex{X^2} - (\ex{X})^2.\]
Just like the variance of a scalar random variable is always positive, the variance of a random Hermitian matrix is always p.d. since it is the expectation of the square of a random matrix. Further, we summarize the information in the matrix variance by a single number called the {\em matrix variance statistic} defined as 
\[\varstat{X} = \| \var{X}\|.\]

As with scalar random variables the variance of the sum of independent random matrices is equal to the sum of the individual variances. 
\begin{fact}
\label{fct:sum}
Given independent random Hermitian matrices $X_1, \ldots, X_n$ and defining $X = $, we have that 
\[\var{\sum_{i=1}^n X_i} = \sum_{i=1}^n \var{X_i},\]
and, so,
\[\varstat{\sum_{i=1}^n X_i} = \left\|\sum_{i=1}^n \var{X_i}\right\|.\]
\end{fact}
This can be proved by a short calculation.

\subsection{Eigenvalue bounds for random matrices}

We now discuss a general method for deriving bounds on the eigenvalues of random Hermitian matrices. The basic method is very much like the Cramer-Chernoff method we studied in the scalar setting (Section~\ref{sec:concentration-scalar:chernoff} with some differences. One of the main differences is that apart from tail bounds we also look to derive bounds on the {\em expectation} of the eigenvalues of random matrices, which is not a relevant consideration in the scalar case where it is generally assumed that the expectation is known. Note that in~\cite{tropp-ftml:2015}, Tropp refers to this Cramer-Chernoff-like method as the ``Laplace transform method.''

\subsubsection{A Cramer-Chernoff-like method for random matrices}
Here we present bounds on the expectation and the tails of eigenvalues of a random matrix.
\begin{proposition}[Tail bounds for eigenvalues]
\label{prp:tail-eval}
Suppose $Y$ is a random Hermitian matrix. Then, for all $t \in \RR$
\begin{align}
\pr{\lmax(Y) \geq t} &\leq \inf_{\theta > 0} e^{-\theta t} \cdot \ex{\tr e^{\theta Y}}, \text{ and} \label{eq:ev-upper-tail}\\
\pr{\lmin(Y) \leq t} &\leq \inf_{\theta < 0} e^{-\theta t} \cdot \ex{\tr e^{\theta Y}}. \label{eq:ev-lower-tail}
\end{align}
\end{proposition}
\begin{proof}
Using the extended Markov inequality (Lemma~\ref{lem:markov}) we know that, for any $\theta > 0$,
\[\pr{\lmax(Y) \geq t} \leq e^{-\theta t} \cdot \ex{e^{\theta \lmax(Y)}}.\]
Since the map $\lmax$ is positive homogenous (Fact~\ref{fct:positive-homogenous}), we know that
\[e^{\theta \lmax(Y)} = e^{\lmax(\theta Y)}. \]
Further, by the Spectral Mapping Theorem (Proposition~\ref{prp:spectral-mapping}) we have that 
\[ e^{\lmax(\theta Y)} = \lmax\left(e^{\theta Y}\right).\]
Since $e^{\theta Y}$ is p.d., the RHS is upper bounded by $\tr e^{\theta Y}$ and we are done with the proof of~\eqref{eq:ev-upper-tail}. The proof of~\eqref{eq:ev-lower-tail} follows similarly by taking $\theta > 0$ and using Fact~\ref{fct:positive-homogenous}(2).
\end{proof}

\begin{proposition}[Expectation bounds for eigenvalues]
\label{prp:ex-eval}
Suppose $Y$ is a random Hermitian matrix. Then
\begin{align}
\ex{\lmax(Y)} &\leq \inf_{\theta > 0} \frac{1}{\theta} \log \ex{\tr e^{\theta Y}}, \text{ and} \label{eq:ev-upper-expectation}\\
\ex{\lmin(Y)} &\geq \sup_{\theta < 0} \frac{1}{\theta} \log \ex{\tr e^{\theta Y}}. \label{eq:ev-lower-expectation}
\end{align}
\end{proposition}
\begin{proof}
Since $\lmax$ is a positive homogenous map (Fact~\ref{fct:positive-homogenous}), for any $\theta > 0$ we can say that 
\[\ex{\lmax(Y)} = \frac{1}{\theta} \ex{\log e^{\lmax(\theta Y})}.\]
By Jensen's inequality
\[ \ex{\log e^{\lmax(\theta Y})} \leq  \log \ex{e^{\lmax(\theta Y)}}.\]
By the Spectral Mapping Theorem (Proposition~\ref{prp:spectral-mapping}) and the fact that $e^{\theta Y}$ is p.d. we have
\[e^{\lmax(\theta Y)} = \lmax\left(e^{\theta Y}\right) \leq \tr e^{\theta Y}.\]
This completes the proof of~\eqref{eq:ev-upper-expectation}. The proof of~\eqref{eq:ev-lower-expectation} follows similarly.
\end{proof}

\subsubsection{Bounds for sums of independent random matrices}
In the scalar case, the cumulants of independent random variables are additive (Proposition~\ref{prp:cumulant-additivity}). However in the matrix case this is not so. For details we refer the reader to Section 3.3 of~\cite{tropp-ftml:2015}. Here we present an alternate route to proving bounds for the eigenvalues of sums of independent random matrices which makes use of the fact that the cumulants of independent random matrices are subadditive.

\begin{lemma}[Cumulant subadditivity]
\label{lem:cumulant-subadditivity}
Given independent random Hermitian matrices $X_1, \ldots, X_n$ of the same dimension, for any $\theta \in \RR$
\[ \ex{\tr \exp\left( \sum_{k=1}^n X_k \right)} \leq \tr \exp \left(\sum_{k=1}^n \log \ex{e^{\theta X_k}}\right).  \]
\end{lemma}
\begin{proof}
The proof relies on a result by Lieb. 
\begin{theorem}[Lieb's theorem]
\label{thm:lieb}
Given a fixed $B \in \HH_d$, the function 
\[f_B(A) = \tr \exp (B + \log A)\]
defined for all p.d. $A \in \HH_d$ is concave. 
\end{theorem}
The proof of Lieb's theorem is non-trivial and can be found in Section 8 of~\cite{tropp-ftml:2015}. Here we state a corollary of this theorem. 
\begin{corollary}
\label{cor:lieb}
Given a fixed $B\in \HH_d$ and a random $X \in \HH_d$ 
\[\ex{\tr \exp (B + X)} \leq \tr \exp \left(B + \log \ex{e^{X}}\right).\]
\end{corollary}
\begin{proof}[Proof of Corollary~\ref{cor:lieb}]
From Theorem~\ref{thm:lieb} we know that the trace exponential function is concave, so by Jensen's inequality we have that 
\[\ex{\tr \exp \left(B + X\right)} = \ex{\tr \exp \left(B + \log e^X\right)}  \leq \tr \exp \left(B + \log \ex{e^X}  \right).\]
\end{proof}
Now, by the tower property of conditional expectations, we have that
\[ \ex{\tr \exp\left( \sum_{k=1}^n X_k \right)} = \exs{X_1,\ldots,X_{n-1}}{\exs{X_n}{\tr \exp\left( \sum_{k=1}^{n-1} X_k + X_n \right) \Big| X_1,\ldots, X_{n-1}}}\]
where the inner expectation is over $X_n$ and the outer expectation is over the remaining random variables. Viewing the inner expectation on its own we see that since it is conditioned on the random variables $X_1, \ldots, X_{n-1}$ those random variables can be considered as fixed. Therefore, by Corollary~\ref{cor:lieb}
\begin{align*}
 \exs{X_n}{\tr \exp\left( \sum_{k=1}^{n-1} X_k + X_n \right) \Big| X_1,\ldots, X_{n-1}}& \leq \tr \exp\left( \sum_{k=1}^{n-1} X_k + \log \exs{X_n}{e^{X_n}\Big| X_1,\ldots, X_{n-1}} \right)\\
& \leq \tr \exp\left( \sum_{k=1}^{n-1} X_k + \log \ex{e^{X_n}} \right).
\end{align*}
We can remove the conditioning on $X_1, \ldots, X_{n-1}$ since the $X_i$ are independent. So, no we have that 
\[\ex{\tr \exp\left( \sum_{k=1}^n X_k \right)} \leq \exs{X_1,\ldots,X_{n-1}}{\tr \exp\left( \sum_{k=1}^{n-1} X_k + \log \ex{e^{X_n}} \right)}.\]
We can now iterate the process we followed above by splitting the expectation on the RHS into a tower of expectations  with the inner expectation on $X_{n-1}$ and the outer one on $X_1, \ldots, X_{n-2}$. Iterating further we get the result.
\end{proof}
Using the subadditivity of cumulants (Lemma~\ref{lem:cumulant-subadditivity}) in Propositions~\ref{prp:tail-eval} and~\ref{prp:ex-eval} we get the following general bounds:
\begin{theorem}[Master bounds for sums of independent random matrices]
\label{thm:master}
Suppose we have independent random matrices $X_1, \ldots, X_n \in \HH_d$, then
\begin{align*}
  \ex{\lmax\left(\sum_{i=1}^n X_i\right)} &\leq \inf_{\theta > 0} \frac{1}{\theta} \log \tr \exp \left(\sum_{i=1}^n \ex{\log e^{\theta X_i}}\right), \text{ and} \\
\ex{\lmin\left(\sum_{i=1}^n X_i\right)} &\geq \sup_{\theta < 0} \frac{1}{\theta}\log \tr \exp \left(\sum_{i=1}^n \ex{\log e^{\theta X_i}}\right). 
\end{align*}
Furthermore, for all $t \in \RR$,
\begin{align*}
\pr{\lmax\left(\sum_{i=1}^n X_i\right) \geq t} &\leq \inf_{\theta > 0} e^{-\theta t} \cdot \tr \exp \left(\sum_{i=1}^n \ex{\log e^{\theta X_i}}\right), \text{ and} \\
\pr{\lmin\left(\sum_{i=1}^n X_i\right) \leq t} &\leq \inf_{\theta < 0} e^{-\theta t} \cdot \tr \exp \left(\sum_{i=1}^n \ex{\log e^{\theta X_i}}\right). 
\end{align*}
\end{theorem}
\subsection{Error estimates for matrix sampling estimators}

Our goal is to use bounds derived from Theorem~\ref{thm:master} to analyze the Random Fourier Features of Rahimi and Recht~\cite{rahimi-nips:2007}. Since the Random Fourier Feauture method falls in the general category of randomized methods that attempt to estimate a matrix by sampling, we will prove bounds on the eigenvalues of sampling-based matrix estimators. The general result that will help us here is the Matrix Bernstein inequality.

\subsubsection{The Matrix Bernstein Inequality}
The scalar Bernstein inequality provides tail bounds for sums of independent bounded centred random variables whose variance can be controlled. In the matrix setting we require a bound on the (spectral) norm of each of the matrices and some control on the matrix variance statistic. In the scalar setting the expectation of such a sum of random variables would naturally be 0, but in the matrix setting we can only hope to get an upper bound on the norm of the sum.

\begin{theorem}[Matrix Bernstein Inequality] 
\label{thm:mbi}
Let $S_{1},\ldots,S_{n}$ be independent random matrices with common dimension $d_1 \times d_2$. Assume that:
\begin{enumerate}
\item $\ex{S_{i}} = 0, 1 \leq i \leq n$, i.e., the matrices are centered.
\item  $\| S_{i} \| \leq L$, $1 \leq i \leq n$, for some $L > 0$, i.e.,  the matrices are norm-bounded.
\end{enumerate}
Let $Z = \sum_{i=1}^{n} S_{i}$ and the matrix variance statistic of $Z$ be
\begin{equation*}
    \varstat{Z}=
    \max \left\{ \left\| \sum_{i=1}^{n} \mathbb{E}(S_{i} S_{i}^{*})\right\|, \left\| \sum_{i=1}^{n} \mathbb{E}(S_{i}^{*} S_{i})\right\| \right\}
\end{equation*}
Then,
\begin{equation*}
    \ex{\|Z\|} \leq \sqrt{2 \varstat{Z} \log (d_{1} + d_{2})} + \frac{L}{3} \log (d_{1} + d_{2}).
\end{equation*}
Also, for all $t \geq 0$, 
\begin{equation*}
  \pr{\|Z\| \geq t} \leq (d_{1} + d_{2}) \exp  \left\{\frac{-t^{2}/2}{\varstat{Z} + \frac{L t}{3}}\right\}.  
\end{equation*}
\end{theorem}
To simplify the presentation we will prove the Matrix Bernstein Inequality only for the case where the matrices are Hermitian. 
\begin{theorem}[Matrix Bernstein Inequality for Hermitian Matrices]
\label{thm:mbi-hermitian}
Let $X_{1},\ldots,X_{n} \in \HH_d$ be independent random hermitian matrices. Assume that
\begin{enumerate}
\item $\ex{X_{i}} = 0, 1 \leq i \leq n$, and
\item  $\| X_{i} \| \leq L$, $1 \leq i \leq n$, for some $L > 0$.
\end{enumerate}
Let $Y = \sum_{k=1}^{n} X_{k}$ and the matrix variance statistic of $y$ be
\begin{equation*}
    \varstat{Y} = \left\| \sum_{k=1}^{n} \ex{X_{k}^{2}}\right\|.
\end{equation*}
Then,
\begin{equation*}
    \ex{\lmax(Y)} \leq \sqrt{2 \varstat{Y} \log d} + \frac{L}{3} \log d.
\end{equation*}
Also, for all t $\geq$ 0,
\begin{equation*}
  \pr{\lmax(Y) \geq t} \leq d \exp \left\{\frac{-t^{2}/2}{\varstat{Y} + \frac{L t}{3}}\right\}.  
\end{equation*}
\end{theorem}
\begin{proof}
The proof is an application of Theorem~\ref{thm:master} with the appropriate bounds for the cumulant generating function plugged in. The bound on the cumulant generating function is the following:
\begin{lemma}
\label{lem:cgf}
Suppose $X$ is a random Hermitian matrix such that $\ex{X} = 0$ and $\lmax(X)\leq L$ for some $L > 0$.
Then for $\theta$ such that $0 < \theta < \frac{3}{L}$,
\begin{equation*}
    \log \ex{e^{\theta X}} \preceq  \left( \frac{\theta^{2}/2}{1- \frac{\theta L}{3}}\right) \ex{X^{2}}.
\end{equation*}
\end{lemma}
\begin{proof}{Proof of Lemma~\ref{lem:cgf}}
For $\theta > 0$ define $f : \RR_+ \to \RR$ as follows:
\begin{equation*}
    f(x) = \frac{e^{\theta x} - \theta x - 1}{x^{2}}
\end{equation*}
We set $f(0) = \frac{\theta^{2}}{2}$ to ensure it is defined everywhere. Since $f^{'}(x) > 0$ for all $x$, the function is increasing and so $f(x) \leq f(L)$ whenever $x \leq L$. Rewriting
\begin{equation*}
    f(x) = \frac{\theta ^ {2}}{2!} + \frac{x\theta ^ {3}}{3!} + \frac{x^{2}\theta ^ {4}}{4!} + \dots
    = \theta^{2} \left[ \frac{1}{2!} + \frac{x \theta} {3!} + \frac{x^{2}\theta ^ {2}}{4!} + \dots \right],
\end{equation*}
and noting that $q! \geq 2\cdot 3^{q-2}$ for all $q \geq 2$, we have that 
\begin{equation*}
    f(x) \leq \frac{\theta^{2}}{2} \sum_{i=0}^{\infty} \left(\frac{x\theta}{3}\right)^i = \frac{\theta^{2}/2}{1-\frac{\theta x}{3}},
\end{equation*}
whenever $x < 3/\theta$. Now, let us extend this analysis to matrices.  Let $X$ be a random Hermitian matrix. 
\begin{equation*}
    e^{\theta X} = I + \theta X + (e^{\theta X} - \theta X - I) = I + \theta X + Xf(X)X.
\end{equation*}
We know that each of the eigenvalues of $X$ is at most $L$. Therefore, by the monotonicity of $f$, we have that $f(\lambda_i) \leq f(L)$ for each $i$, where $\lambda_1,\ldots \lambda_d$ are the eigenvalues of $X$. Therefore $f(L)I$, each of whose eigenvalues is $f(L)$ dominates $f(X)$ in the p.d. partial order. 
So, we have that 
\begin{equation*}
    e^{\theta X} \preceq 
    I + \theta X + f(L) X^{2} \preceq I + \theta X + \frac{\theta^{2}/2}{1 - \frac{\theta L}{3}} X^{2}.
\end{equation*}
Taking expectations on both sides, recalling that expectation preserves the p.d. partial order, and noting that $\ex{X} = 0$, we have
\begin{equation*}
    \ex{e^{\theta X}}
    \preceq
    I + \frac{\theta ^{2}/2}{1 - \frac{\theta L}{3}} \ex{X^{2}}.
\end{equation*}
Since $1 + a \leq e^{a}$, for all $a \in \mathbb{R}$. Therefore,
\begin{equation*}
    \ex{e^{\theta X}} \preceq \exp \left\{\frac{\theta ^{2}/2}{1 - \frac{\theta L}{3}} \ex{X^{2}}\right\}.
\end{equation*}
Taking logs on both sides and recalling that log is an operator monotone function (Proposition~\ref{prp:log-monotone}) we get the result. 
\end{proof}
Now we turn to using the obtained bound to derive eigenvalue bounds. For brevity we will say that
\begin{equation*}
    g(\theta) = \frac{\theta^{2}/2}{1 - \frac{\theta L}{3}}
\end{equation*}
Since the trace exponential is monotonic in the p.d. partial order (Corollary~\ref{cor:trexp}), we can use Lemma~\ref{lem:cgf} with the master bounds (Theorem~\ref{thm:master}) to obtain:
\begin{equation*}
    \ex{\lmax(Y)}
    \leq \inf_{0 < \theta < \frac{3}{L}} \frac{1}{\theta}\log \tr \left[\exp \left\{g(\theta) \sum_{k = 1}^{n} \ex{X_{k}^{2}}\right\}\right].
\end{equation*}
Note that range of $\theta$ is curtailed to the $[0,3/L]$ since the upper bound of Lemma~\ref{lem:cgf} only applies in this range.
In the next step we use $\mathbb{E} [Y^{2}] = \sum_{k = 1}^{n} \mathbb{E} [X_{k}^{2}]$ (Fact~\ref{fct:sum}) and also the fact that the trace of a p.d. matrix $A \in \HH_d$ is upper bounded by $d \lmax(A)$ to obtain
\begin{equation*}
    \ex{\lambda_{max}(Y)} \leq 
    \inf_{0 < \theta < \frac{3}{L}} \frac{1}{\theta}
    \log \left[
    d  \lmax\left(
    \exp \left\{ 
    g(\theta) \ex{Y^{2}}
    \right\}\right)
    \right].
\end{equation*}
We can move $\lmax$ inwards using the Spectral Mapping Theorem (Proposition~\ref{prp:spectral-mapping}) to get
\begin{equation*}
    \ex{\lambda_{max}(Y)} \leq 
    \inf_{0 < \theta < \frac{3}{L}} \frac{1}{\theta}
    \log \left[
    d 
    \exp \left\{ 
    g(\theta) 
    \lmax\left( 
    \ex{Y^{2}}\right)
    \right\}
    \right].
\end{equation*}
We note that since $Y^2$ is p.d.,  $\lmax(\ex{Y^2})$ is precisely $\varstat{Y}$. The rest of the proof for the expectation bound simply involves differentiating the RHS of the above equation and finding the minima. We skip those steps. 

For the tail bound we follow steps similar to those for the expectation bound to obtain
\begin{equation*}
    \pr{\lambda_{max}(Y) \geq t}
    \leq
    \inf_{0 < \theta < 3/L} 
    d  
    e^{-\theta t} \left[
    \exp \left\{
    g(\theta)
    \sum_{k=1}^{n} 
    \log \ex{Y^{2}}
    \right\}
    \right].
\end{equation*}
\end{proof}

\subsubsection{Matrix sampling estimators}

The method of Rahimi and Recht~\cite{rahimi-nips:2007} falls under the broad category of matrix sampling estimators. In this setting we have a target matrix $B$ that we want to estimate and a random matrix $R$ such that $\ex{R} = B$. We repeatedly sample instances of $R$ independently $n$ times, call them $R_1, \ldots, R_n$, and use the emprirical estimate $(\sum_{i=1}^n R_i)/n$ as an estimate of $B$. The Matrix Bernstein inequality allows us to determine what the error in this estimate is. As expected, the error depends on the variance of $R$. We state this result as a corollary.

\begin{corollary}
\label{cor:error-estimate}
Let $B$ be a fixed $d_1 \times d_2$ matrix. Suppose that $R$ is  $d_1 \times d_2$ random matrix such that $ \ex{R} = B$ and $ \| R \| \leq L$. Let 
$$ m_2(R) = \max \{ \lVert \ex{RR^{*}} \rVert , \lVert \ex{R^{*}R}\rVert \}$$
be the {\em per sample second moment} of $R$ and let the matrix sampling estimator be:
$$ \bar{R}_n = \frac{\sum_{i=1}^{n} R_k}{n}$$
where each $R_k$ is an independent copy of $R$. Then
$$ \ex{\lVert \bar{R}_n - B \rVert} \leq \sqrt{\frac{2m_2(R)\log(d_1+d_2)}{n}} + \frac{2L\log(d_1+d_2)}{3n}.$$
Furthermore for all $t \geq 0$
$$ \pr{\lVert \bar{R}_n - B \rVert \geq t} \leq (d_1 + d_2)\exp\left\{{\frac{-nt^2/2}{m_2(R)+\frac{2Lt}{3}}}\right\}.$$
\end{corollary}
\begin{proof}
Let $S_i = (R_i - \ex{R})/n$ and $Z = \bar{R}_n - B = \sum_{i=1}^{n} S_i$. 
Each $S_i$ has zero mean, and is identically and independently distributed. We bound the norm of the $S_i$ by observing that 
$$ \lVert S_i \rVert \leq \frac{\lVert R_i \rVert + \lVert \ex{R} \rVert}{n} \leq \frac{\lVert R_i \rVert + \ex{\lVert R\rVert}}{n} \leq \frac{2L}{n},$$
where the first inequality is the triangle inequality of the spectral norm, the second results from Jensen's inequality and the last from our assumption that $ \lVert R \rVert \leq L$.

To control the matrix variance statistic note that, since all the $S_i$ are identical,
$$ \varstat{Z} = \max \{ \left\| \sum_{i=1}^{n} \ex{S_{i}{S_{i}}^{*}} \right\| ,\left\| \sum_{i=1}^{n} \ex{{S_{i}}^{*} S_{i}} \right\| \}  = n \max \{ \left\|  \ex{S_{1}{S_{1}}^{*}} \right\| ,\left\|  \ex{{S_{1}}^{*} S_{1}} \right\| \}$$
Now
$$ \ex{S_1 {S_1}^{*}} = \frac{\ex{RR^*- \ex{R}\ex{R}^*}}{n^2} \preceq \frac{\ex{RR^*}}{n^2},$$
where we ignore the second term $\ex{R}\ex{R}^* = BB^*$ since it is p.d. As a consequence,
$$ \varstat{Z} =  n \lVert \ex{S_1 {S_1}^{*}} \rVert \leq n \frac{\lVert \ex{RR}^* \rVert}{n^2} = \frac{m_2(R)}{n}.$$
Substituting the upper bounds on $\|S_i\|$ and $\varstat{Z}$ in the statement of Theorem~\ref{thm:mbi} gives us the result.
\end{proof}

\subsection{An error estimate for Random Fourier Features}

Let us return to Random Fourier Features (Definition~\ref{def:rff}). We now view the construction as a method for estimating the Gram Matrix $G$ of the Gaussian Radial Basis function and give a bound on the norm of the error. To do so we reformulate the Random Fourier Feature definition in terms that will allow us to apply Corollary~\ref{cor:error-estimate}.
\begin{example}[A norm bound for the error of Random Fourier Feautures]
\label{ex:rff-mbi}
Given a set of training points $T= \{\bm{x}_1,\ldots,\bm{x}_n\} \subseteq \mathbb{R}^d$ such that $d>1$. For some $\sigma \in \RR$ and $D>0$, we construct a random feature map $\phi: \mathbb{R}^d \to [0,1]^D$ as given in Definition~\ref{def:rff}. Define an inner product on $[0,1]^D$ as follows: $\langle \bm{x}, \bm{y}\rangle = \frac{1}{D} \sum_{k=1}^D \bm{x}(k)\bm{y}(k)$. Then, if $G$ is an $n \times n$ matrix such that 
$$G_{ij} = e^{-\frac{\|\bm{x}_i - \bm{x}_j\|_2^2}{2\sigma^2}}$$
and $\bar{R}_D$ is an $n \times n$ matrix with $(i,j)$th entry $\langle \bm{x}_i, \bm{x}_j\rangle$ then  
$$\frac{\ex{\lVert \bar{R}_D - G \rVert}}{\lVert G \rVert} \leq \varepsilon$$
if
\[D \geq \frac{16\log(2n)\mathrm{intDim}(G)}{\varepsilon^2}.\]
\end{example}
\begin{proof}
Given training set $T = \{ \bm{x}_1,\bm{x}_2,\ldots,\bm{x}_n \} \subseteq \RR^d$ we define a random vector $\bm{z}$ as follows: Choose $\bm{w} \in \RR^d$ according to distribution $\mathrm{N}(0,\sigma^2I)$ and $b$ uniformly at random in $[0,2\pi]$, then set $\bm{z}_i = \sqrt{2}\cos(\bm{w}^t\bm{x}_i + b)$. 
Now form a matrix $R = \bm{z}\bm{z}^{*}$. From Fact~\ref{fct:rr} we know that $\ex{R} = G$. 

Now we can say that Definition~\ref{def:rff} amounts to picking $D$ iid matrices $R_1,\ldots, R_D$ with the same distribution as $R$. Our estimator for $G$ is  $\bar{R} = \sum_{i=1}^{D} R_i/D$. In order to apply Corollary~\ref{cor:error-estimate} we need bounds on $\lVert R\rVert$ and $m_2(R)$. 
To bound $\lVert R\rVert$ we note that, since $R = \bm{z}\bm{z}^{*}$, $\lVert R\bm{x}\rVert$ is maximimzed if $\bm{x}$ is a unit vector in the direction of $\bm{z}$, i.e., 
$$ \lVert R \rVert = \max_{\lVert \bm{x} \rVert = 1} \lVert \bm{z}\bm{z}^{*}\bm{x} \rVert  = \left\lVert \bm{z}\bm{z}^{*}\frac{\bm{z}}{\lVert \bm{z} \rVert} \right\rVert = {\lVert \bm{z} \rVert}^2 \leq 2n. $$
To bound $m_2(R)$ note that 
$$ \ex{\lVert R^2\rVert} = \ex{\lVert\bm{z}\bm{z}^{*}\bm{z}\bm{z}^{*}\rVert} = \ex{{\lVert \bm{z} \rVert}^2 \lVert \bm{z}\bm{z}^{*}\rVert} \leq 2n\ex{\lVert\bm{z}\bm{z}^{*}\rVert} = 2n\lVert G\rVert.$$
Where the last inequality comes from the fact that $\bm{z}\bm{z}^* = R$ and $\ex{R} = G$. 

Substituting these in Corollary~\ref{cor:error-estimate} we get 
$$ \ex{\lVert \bar{R} - G \rVert} \leq \sqrt{\frac{4n \lVert G \rVert \log(2n)}{D}} + \frac{4n\log(2n)}{3D}$$
Dividing both sides by $ \lVert G \rVert$, we have 
$$ \frac{\ex{\lVert \bar{R} - G \rVert}}{\lVert G \rVert} \leq \sqrt{\frac{4n\log(2n)}{D\lVert G \rVert }} + \frac{4n\log(2n)}{3D \lVert G \rVert }$$
Since the diagonal elements of $G$ are all equal to 1, we know that $\tr G = n$. We therefore identify $n/\lVert G \rVert$ as the intrinsic dimension of $G$ (Definition~\ref{def:intdim}). So we see that 
$$ \frac{\ex{\lVert \bar{R} - G \rVert}}{\lVert G \rVert} \leq \sqrt{\frac{4\text{ intDim}(G)\log(2n)}{D}} + \frac{4\text{ intDim}(G)\log(2n)}{3D}$$
Now, for any $1>\varepsilon>0$, if we set $$D \geq \frac{4\log(2n)\text{intDim}(G)}{\varepsilon^2},$$ we get that
$$ \frac{\ex{\lVert \bar{R}_D - G \rVert}}{\lVert G \rVert} \leq \varepsilon + \frac{\varepsilon^2}{3} \leq 2\varepsilon.$$
\end{proof}

\noindent{\bf Discussion.} Comparing the results of Example~\ref{ex:rff-hoeffding} and Example~\ref{ex:rff-mbi} we see that in the latter we have $\mathrm{intdim}(G)$ in place of $n$. The intrinsic dimension of $G$ is clearly upper bounded by $n$ but could, in practice, be much lower. For example if all the $n$ training points are extremely closely clustered the intrinsic dimension could be as low as 1 (or just a little more than 1). 

It may be argued that the lower bound on $D$ provided in Example~\ref{ex:rff-hoeffding} ensures that {\em every} entry of $\bar{R}_D$ is close to the corresponding entry of $G$. However, if we look at the dual formulation of the support vector machine (Problem~\ref{prb:svm-dual}), we note that the second term in the objective function is of the form $\bm{a}^TG\bm{a}$ where $\bm{a}_i = \alpha_iy_i$. In such a situation an error bound on the norm of the approximation error of the Random Fourier Feature approximation may be more useful in estimating the convergence time of an algorithm that seeks to solve the computational problem.

\bibliographystyle{plain}
\bibliography{kernels}

\end{document}